\pdfoutput=1

\documentclass[]{fairmeta}

\usepackage{graphicx} 
\usepackage{natbib}  
\usepackage{caption} 
\usepackage{algorithm}
\usepackage{algorithmic}
\usepackage{amsmath}
\usepackage{amssymb}
\usepackage{amsthm}
\usepackage{booktabs}
\usepackage[table]{xcolor}

\usepackage{pgfplots}
\pgfplotsset{compat=1.16}
\usepgfplotslibrary{fillbetween}

\usepackage{listings}\usepackage{xcolor}
\theoremstyle{plain}

\newcommand{\K}{\mathcal{K}}
\newcommand{\Env}{\mathcal{E}}
\newcommand{\Lib}{\mathcal{L}}
\newcommand{\Map}{\mathcal{M}}
\newcommand{\cm}{\rho}
\newcommand{\pol}{\pi}
\newcommand{\seq}[2]{#1 \vdash #2}
\newcommand{\concl}{\mathrm{concl}}

\PassOptionsToPackage{capitalise}{cleveref}
\usepackage[most]{tcolorbox} 
\usepackage{array}
\usepackage{listings}
\usepackage{needspace}
\lstdefinestyle{promptbody}{%
  basicstyle=\ttfamily\footnotesize,%
  breaklines=true,%
  breakatwhitespace=true,%
  columns=fullflexible,%
  keepspaces=true,%
}
\lstdefinelanguage{Lean}{%
  morekeywords={theorem,lemma,def,structure,by,intro,obtain,have,haveI,exact,calc,%
    rw,simp,induction,with,constructor,by_contra,push_neg,omega,ring,ring_nf,%
    linarith,norm_num,field_simp,decide,use,open,import,set_option,Type,Prop,fun,True},%
  sensitive=true, morecomment=[n]{/-}{-/}, morecomment=[l]{--}, morestring=[b]",%
}
\lstdefinestyle{lean}{language=Lean, basicstyle=\ttfamily\footnotesize,%
  backgroundcolor=\color{black!3}, xleftmargin=3pt, framexleftmargin=3pt,%
  keywordstyle=\color{blue!60!black}, commentstyle=\color{green!45!black},%
  columns=fullflexible, keepspaces=true, breaklines=true, showstringspaces=false,%
  literate=%
   {∀}{{$\forall$}}1 {∃}{{$\exists$}}1 {∈}{{$\in$}}1 {∉}{{$\notin$}}1
   {∧}{{$\wedge$}}1 {∨}{{$\vee$}}1 {¬}{{$\neg$}}1 {→}{{$\to$}}1 {↦}{{$\mapsto$}}1
   {↪}{{$\hookrightarrow$}}1 {↑}{{$\uparrow$}}1 {≠}{{$\neq$}}1 {≥}{{$\geq$}}1
   {≤}{{$\leq$}}1 {⟨}{{$\langle$}}1 {⟩}{{$\rangle$}}1 {ℕ}{{$\mathbb{N}$}}1
   {ℝ}{{$\mathbb{R}$}}1 {ℤ}{{$\mathbb{Z}$}}1 {≡}{{$\equiv$}}1 {∑}{{$\sum$}}1
   {∖}{{$\setminus$}}1 {σ}{{$\sigma$}}1 {α}{{$\alpha$}}1 {·}{{$\cdot$}}1 {∘}{{$\circ$}}1  {←}{{$\leftarrow$}}1 {⊢}{{$\vdash$}}1 {⁻¹}{{$^{-1}$}}2
}
\usepackage[noabbrev,nameinlink]{cleveref} 
\usepackage{thmtools}
\usepackage{thm-restate}

\newtcblisting[auto counter]{PromptBox}[2][]{%
  breakable,enhanced,listing only,listing engine=listings,
  listing options={style=promptbody},
  colback=gray!6,colframe=gray!45,coltitle=white,colbacktitle=black!66,
  fonttitle=\bfseries\small,boxrule=0.5pt,arc=1.5pt,
  left=5pt,right=5pt,top=5pt,bottom=5pt,
  before skip=6pt,after skip=6pt,
  title={Prompt~\thetcbcounter: #2},
  title after break={Prompt~\thetcbcounter: #2 (continued)},
  #1,
}
\crefname{tcb@cnt@PromptBox}{prompt}{prompts}
\Crefname{tcb@cnt@PromptBox}{Prompt}{Prompts}
\newtcolorbox[auto counter]{ExampleBox}[2][]{%
  enhanced,breakable,colback=white,colframe=black!55,coltitle=white,colbacktitle=black!66,
  fonttitle=\bfseries\small,fontupper=\small,boxrule=0.5pt,arc=1.5pt,
  left=5pt,right=5pt,top=4pt,bottom=4pt,before skip=7pt,after skip=7pt,
  title={Example~\thetcbcounter.\ #2},
  title after break={Example~\thetcbcounter.\ #2 (continued)},
  pad at break*=3pt,#1}
\newtcblisting[auto counter]{LeanBox}[2][]{%
  breakable,enhanced,listing only,listing engine=listings,
  listing options={style=lean},
  colback=white,colframe=black!45,coltitle=white,colbacktitle=black!66,
  fonttitle=\bfseries\small,boxrule=0.5pt,arc=1.5pt,
  left=5pt,right=5pt,top=3pt,bottom=3pt,
  before skip=7pt,after skip=4pt,
  title={Listing~\thetcbcounter:\ #2},
  title after break={Listing~\thetcbcounter:\ #2 (continued)},
  #1}
\crefname{tcb@cnt@LeanBox}{listing}{listings}
\Crefname{tcb@cnt@LeanBox}{Listing}{Listings}
\crefname{tcb@cnt@ExampleBox}{example}{examples}
\Crefname{tcb@cnt@ExampleBox}{Example}{Examples}
\newcommand{\ExampleNeedspace}{\Needspace{0.34\textheight}}

\title{ProofEvolve: Neuro-Symbolic Evolution for Formal Automated Theorem Proving}

\author[1,2,*]{Wenqian Ye}
\author[2,*]{Ziwei Guan}
\author[1]{Eric Xie}
\author[1]{Bohan Liu}
\author[2]{Shivani Modi}
\author[2]{Buyun Zhang}
\author[2]{Ellie Dingqiao Wen}
\author[1]{Henry Kautz}
\author[1]{Aidong Zhang}

\affiliation[1]{University of Virginia}
\affiliation[2]{Meta AI}

\contribution[*]{Core contributors}

\abstract{Automated theorem proving offers a natural foundation for recursive self-improvement in scientific discovery. However, existing neural provers do not fully preserve this recursive structure, where the learning process should be self-improving over time.  Existing methods either embed proof experience into model parameters through expensive weight updates, or keep verified intermediate deductions only within the current problem.  In addition, these methods also heavily rely on sparse whole-proof feedback, even when unsuccessful partial attempts contain useful discoveries. To close the gap, we propose \textbf{ProofEvolve}, a \textit{neuro-symbolic} framework that evolves explicit, formally verified \textit{symbolic} proof structures with \textit{neural} models to decisively expand the knowledge boundary. In this framework, the neural model proposes variation operators, including decompositions, repairs, and schema recombinations. The symbolic Lean kernel verifies every proof transition. Over the evolution loops, ProofEvolve computes verified closure over the resulting proof directed acyclic graphs (DAGs). Within each problem, ProofEvolve evolves partial AND-OR proof DAGs in a behaviorally indexed archive. Across problems, kernel-checked schema extraction adds newly proved sub-DAGs to a persistent schema library. Proof DAGs inherit the solved results through typed schema recombination, with every residual premise exposed as a new subgoal. This evolutionary process preserves verified results from incomplete attempts and makes them available for later proofs without weakening formal soundness. Across three competition-level Lean benchmarks, ProofEvolve achieves the highest average solve rate among the evaluated proof systems.
}

\date{August 10, 2026}
\correspondence{Wenqian Ye and Aidong Zhang at
\href{mailto:wenqian@virginia.edu,aidong@virginia.edu}{\texttt{\{wenqian, aidong\}@virginia.edu}}}

\begin{document}

\maketitle

\section{Introduction}

Theorem proving paves the way for scientific reasoning to discover new knowledge with formal guarantee of correctness. A proof often contributes more than its stated conclusion. It can produce reusable lemmas, reveal hidden structures, and expose new hypotheses. Even unsuccessful proof programs can push forward important advances. For centuries, mathematicians attempted to derive \textit{Euclid's parallel
postulate} from his remaining axioms. Examining geometries in which the
postulate does not hold instead led to the development of non-Euclidean
geometry \citep{bonola1955noneuclidean}. A similar pattern is in
theoretical computer science (TCS). Hilbert's \textit{Entscheidungsproblem} asked
whether a general procedure could determine the validity of any
statement in first-order logic \citep{hilbert1928grundzuege}. Following the unsolved problem, Church
and Turing proved that no such general procedure exists
\citep{church1936unsolvable,turing1937computable}. Turing's analysis
also introduced a formal model of computation that later became the foundation of TCS. Therefore, proofs fundamentally convert individual results into reusable knowledge that supports later discoveries.

\begin{figure}[!ht]
\centering
\includegraphics[width=\textwidth]{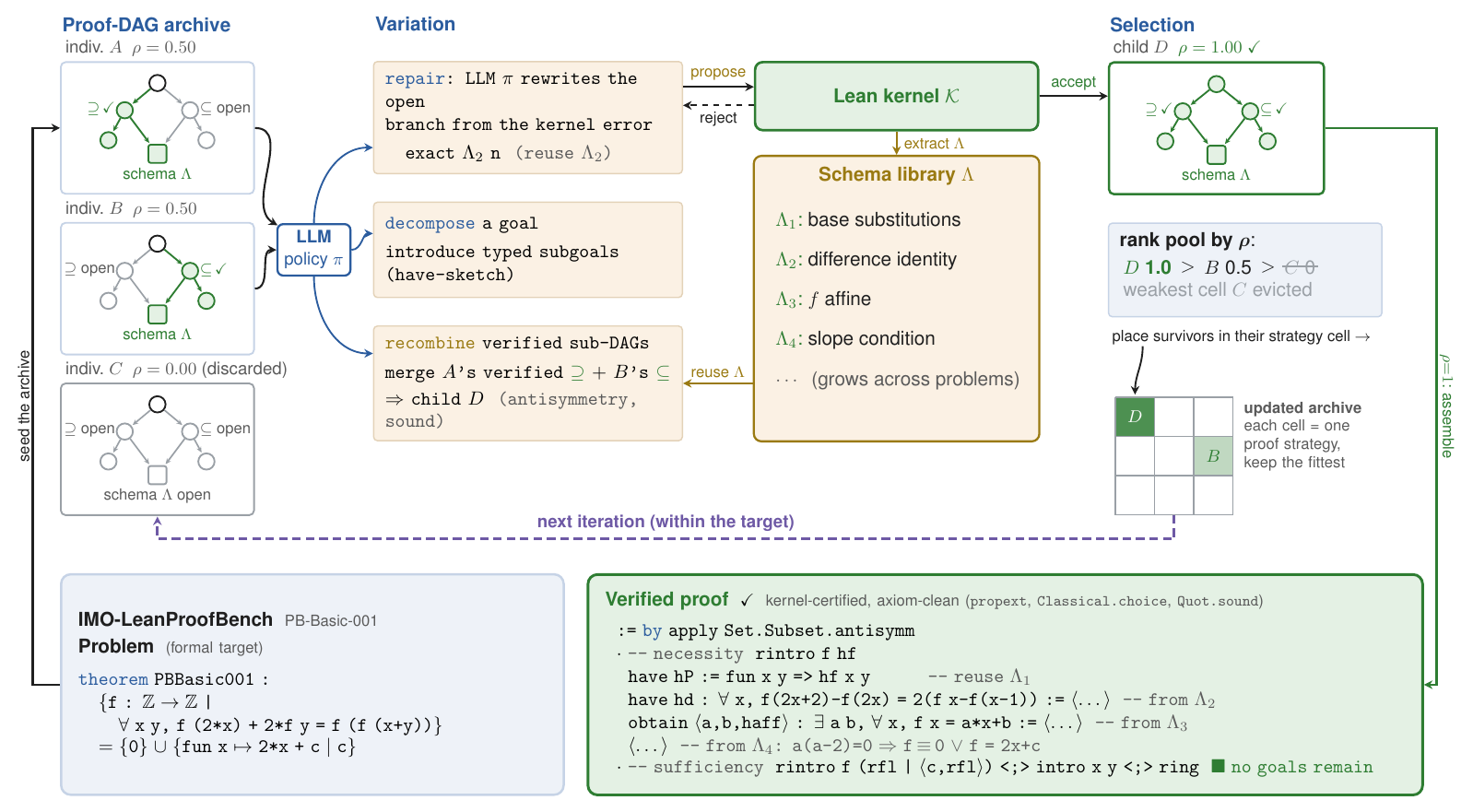}
\caption{\textbf{Overview of ProofEvolve.} Within a target, a proof DAG archive preserves structurally diverse candidates and ranks them by verified closure $\rho$, computed from kernel-accepted subgoals. Across targets, checked extraction adds newly proved results to a persistent schema library, and typed recombination instantiates them in later DAGs while exposing all residual premises as subgoals.}
\vskip -10pt
\label{fig:overview}
\end{figure}

Recently, AI systems have begun to automate parts of scientific discovery. Trained neural models with automatic evaluation have produced new algorithms, mathematical constructions, and structural patterns
\citep{romeraparedes2024funsearch,novikov2025alphaevolve,davies2021advancing}.
Evolutionary search is also promising as it can improve candidates over time through repeated generation, evaluation, and selection. However, current linguistic-based neural theorem provers still heavily use a small part of the verified structure. Training-based systems store experience mainly in model parameters, so new results affect later problems only after another training cycle
\citep{alphaproof2025nature,ren2025deepseekproverv2,route2025goedelproverv2}.
Agentic systems leverage multi-agent collaboration to work on problems with subgoals during inference
\citep{dsp,hilbert}.
One latest work, LEAP \citep{leap}, also shares intermediate lemmas across branches of a proof DAG, but this memory remains tied to the current target. Thus, current systems still focus mainly on whether the root theorem is solved.

Formal theorem proving in Lean~4 \citep{demoura2021lean4} provides a reliable setting for cumulative evolution because every accepted proof step is rigorously verified. However, current neural theorem provers do not fully preserve verified progress. Training based methods store experience mainly in model parameters \citep{alphaproof2025nature,ren2025deepseekproverv2,route2025goedelproverv2}, while agentic methods reuse information mainly within the current problem \citep{dsp,hilbert,leap}. As a result, useful structures from partial and completed proofs are rarely inherited across searches. This is a central limitation for evolutionary search, which requires useful structures to survive unsuccessful candidates and pass to later generations. A small textual mutation may invalidate a complete proof even when much of its verified argument remains correct \citep{nagashima2019evoisabelle}. A self-improving prover should therefore evolve verified partial structures rather than complete proof texts alone and preserve them for recombination across problems.

To address these shortcomings, we propose \textbf{ProofEvolve}, a neuro-symbolic evolutionary framework in which improvement grows in an explicit body of verified symbolic structures with neural proposals. As shown in Figure~\ref{fig:overview}, ProofEvolve grows a proof DAG for each theorem through kernel-verified transitions. At each step the language model proposes a decomposition, repair, or schema recombination, and Lean~4 kernel accepts or rejects it. A behaviorally indexed archive \citep{mouret2015mapelites} keeps structurally diverse partial proofs and ranks them by \emph{verified closure}, a kernel-grounded score that aggregates proved subgoals through the AND-OR structure, so selection acts on graded progress rather than a single pass-or-fail verdict. Across problems, ProofEvolve extracts closed sub-DAGs as theorem schemas. Typed recombination later instantiates a schema at a matching goal, exposes any remaining premises as new subgoals, and re-checks the result with the kernel. Verified results from earlier problems therefore enlarge the reachable search space of later ones.

To show the effectiveness, we evaluate our framework on three competition-level benchmarks, PutnamBench, IMO-LeanProofBench, and CombiBench, with a frontier LLM (Claude Opus 4.8) as the base model under matched per-target budgets, and on a disjoint Lean~Workbook split with the frontier open-weight model Qwen3.5-397B-A17B-FP8. ProofEvolve achieves an average solve rate of $57.8\%$, compared with $50.5\%$ for LEAP and $45.9\%$ for Hilbert. Ablation and test-time scaling studies further analyze the method's components and its behavior as the per-target budget grows. A separate evaluation isolates the library itself: on $744$ Lean~Workbook theorems disjoint from the $9{,}968$-theorem source stream, $5{,}546$ of the prover's own kernel-checked proofs improve the solve rate by about four points over zero-shot, whereas random retrieval from the same library gives no improvement.

In summary, our main contributions are as follows:

\begin{list}{$\bullet$}{\setlength{\leftmargin}{1.4em}\setlength{\labelwidth}{1.0em}\setlength{\labelsep}{0.4em}\setlength{\itemsep}{3pt}\setlength{\topsep}{3pt}\setlength{\parsep}{0pt}}
\item We formulate theorem proving as an evolution process over structured symbolic knowledge with persistent inheritance through a Lean-verified theorem schema library, where the neural models propose new directions to extend the knowledge boundary.
\item We introduce verified closure as a kernel-grounded fitness measurement and typed schema recombination as a mechanism to enable verified subproofs across targets.
\item Extensive experiments on ProofEvolve against state-of-the-art neural models and agentic baselines on three competition-level Lean benchmarks show the significant effectiveness of the proposed framework, and on $744$ Lean~Workbook theorems disjoint from its library the prover's own verified proofs add about four points over zero-shot, while random retrieval from the same library adds nothing.
\end{list}

\section{Related Work}

\subsection{Neural theorem provers}

Neural theorem provers train an LLM policy to propose tactics or complete proofs in a formal language. GPT-f \citep{polu2020gptf} and PACT \citep{han2021pact} established language-model policies for tactic generation. Other works \citep{yang2023leandojo,lample2022hypertree,xin2025bfsprover,xin2025bfsproverv2} combine learned proposals with retrieval or tree search. Recent systems obtain stronger policies from synthetic proofs, supervised fine-tuning, and reinforcement learning
\citep{ren2025deepseekproverv2,route2025goedelproverv2,kiminaprover2025,seedprover2025,ji2025leanabellv2}. AlphaProof \citep{alphaproof2025nature} trains at scale on kernel-checked self-generated experience and performs test-time adaptation on difficult targets. Retrieval-augmented provers reuse existing library lemmas at inference time \citep{shen2025realprover}, while self-play systems learn from both successful and failed proof trees \citep{poesia2024minimo}. ProofEvolve instead keeps the model fixed and stores newly proved results as explicit, verified schemas.

\subsection{Agentic theorem provers}

Agentic systems use multiple language models to structure proof search through decomposition, retrieval, and verifier feedback. Draft-Sketch-Prove \citep{dsp} turns informal arguments into formal sketches. COPRA \citep{copra} constructs proofs through repeated tactic execution. Hilbert \citep{hilbert} recursively decomposes difficult goals and repairs failed proofs, while LEAP \citep{leap} uses an AND-OR proof DAG to share intermediate lemmas across branches. AlphaProof Nexus \citep{apn2026} applies compiler-guided agents to research problems, while AlphaGeometry \citep{trinh2024alphageometry} combines neural proposals with symbolic deduction in geometry. These methods strengthen search within a target. In contrast, ProofEvolve additionally extracts verified sub-DAGs for reuse across targets, which enables the framework the ability to recursively evolve.

\subsection{Symbolic knowledge evolution}

Evolutionary program search, including FunSearch \citep{romeraparedes2024funsearch} and AlphaEvolve \citep{novikov2025alphaevolve}, alternates language model generation with automatic evaluation and selection, while MAP-Elites preserves diverse high quality candidates across behavioral niches \citep{mouret2015mapelites}. LEGO-Prover expands a lemma library \citep{wang2023legoprover}, although storage alone does not guarantee reuse \citep{legolearnfail2025}. DreamProver is closest to our setting because it keeps the model fixed and builds a transferable Lean library through wake sleep abstraction \citep{zhang2026dreamprover,ellis2021dreamcoder}. Other methods retrain a retriever \citep{kumarappan2024leanagent}, distill proof strategies \citep{fang2025strat2rocq}, or modify agent code using empirical fitness \citep{zhang2025darwingodel}. Formal proof evolution is difficult because verification is binary and small edits can invalidate useful candidates \citep{nagashima2019evoisabelle}. ProofEvolve instead combines verified closure and typed schema recombination in a single inference time process. Partial proof DAGs are selected through kernel accepted subgoals, while newly closed sub-DAGs become schemas reused in later candidates without modifying the model. Under Kautz's taxonomy on Neural Symbolic AI \citep{kautz2022thirdaisummer}, ProofEvolve is a \textsc{Neuro[Symbolic]} system where the Lean kernel and verified closure operator are embedded in the neural generation process to recursively improve over time.

\section{Preliminaries}
\label{sec:prelim}

\noindent
We first define the symbolic structures on which ProofEvolve operates. Every state is a Lean~4 artifact, and the search retains each partial result rather than discarding it. A closed fragment of an attempt is a valid, reusable result even before the attempt closes its root goal (Eq.~\eqref{eq:assembly}). This property is what later lets the search accumulate and transfer verified work (Section~\ref{sec:method}).

\paragraph{Lean verification.}
We work in a fixed Lean~4 environment $\Env$ that imports a matched Mathlib version \citep{mathlib2020}. We write $\Env;\Gamma\vdash_{\K}p:g$ when the term $p$ elaborates under local context $\Gamma$. It contains no unresolved metavariables or placeholders, and is accepted at type $g$ by Lean's kernel $\K$. A tactic state is $s=(\seq{\Gamma}{g})$ \citep{polu2020gptf,yang2023leandojo}, and its checked witnesses form
\begin{equation}
\mathsf{Prf}_{\Env}(s)
=\{p\mid \Env;\Gamma\vdash_{\K}p:g\}.
\label{eq:proof-set}
\end{equation}
We write $u\equiv_{\Env}v$ for definitional equality in $\Env$, let $\mathsf{Thm}(\Env)$ denote the theorem declarations imported into $\Env$, and call a target $T$ \emph{well-formed} when $\Env;\varnothing\vdash_{\K}T:\mathsf{Prop}$.

\paragraph{Reusable proof structures.}
For a \textit{well-formed} target $T$, a proof attempt is a finite acyclic AND-OR proof DAG $D=(V,E,r)$ with root $r=(\seq{\Gamma_0}{T})$, where each node is a tactic state. An accepted hyperedge $e=(s;s_1,\ldots,s_k)$ is one Lean-elaborated proof constructor from the child obligations to the source obligation, i.e., a checked realizer
\begin{equation}
F_e:\prod_{i=1}^{k}\mathsf{Prf}_{\Env}(s_i)
\longrightarrow \mathsf{Prf}_{\Env}(s).
\label{eq:realizer}
\end{equation}
For $k=0$, the empty product is the singleton $\{\star\}$ and $F_e(\star)\in\mathsf{Prf}_{\Env}(s)$ is a checked closing witness. Multiple edges leaving $s$ are alternative proof steps, whereas the children of a single edge must all be discharged. For DAGs sharing a root, we write $D\preceq D'$ when $V(D)\subseteq V(D')$, $E(D)\subseteq E(D')$, and all existing node labels and edge realizers are preserved, and we call $D'$ an \emph{extension} of $D$.

Let $\mathrm{out}_D(s)$ be the accepted edges leaving $s$ and $\mathrm{ch}(e)$ be their children. Since $D$ is acyclic, closure is well defined by
\begin{equation}
\mathrm{Closed}_D(s)
\Longleftrightarrow
\exists e\in\mathrm{out}_D(s)\;
\forall s'\in\mathrm{ch}(e),\ \mathrm{Closed}_D(s').
\label{eq:closed}
\end{equation}
The universal condition is vacuous for a checked closing edge. For each closed $s$, we fix one witnessing edge $\mathrm{win}_D(s)$, and the constructions below hold for any such choice. The open search boundary is
\begin{equation}
\mathrm{frontier}(D)=
\{s\in V:\neg\mathrm{Closed}_D(s),\
\mathrm{out}_D(s)=\varnothing\}.
\label{eq:frontier}
\end{equation}
Writing $\mathrm{win}_D(s)=(s;s_1,\ldots,s_k)$ and taking the empty tuple to be $\star$, proof assembly is the recursion
\begin{equation}
\mathrm{Asm}_D(s)=
F_{\mathrm{win}_D(s)}
\left((\mathrm{Asm}_D(s_i))_{i=1}^{k}\right).
\label{eq:assembly}
\end{equation}
Acyclicity makes this recursion well founded, and Eq.~\eqref{eq:realizer} gives $\mathrm{Asm}_D(s)\in\mathsf{Prf}_{\Env}(s)$: a closed internal sub-DAG is a typed result, reusable while its root remains open.

\section{Neuro-Symbolic Evolution for Formal Automated Theorem Proving}
\label{sec:method}

\subsection{Overview}
\label{sec:state}

\noindent
Figure~\ref{fig:overview} shows the ProofEvolve pipeline. A neural language model
proposes structural variations, the Lean kernel checks every proof transition and
schema application, and the search keeps the checked structures that make the most
verified progress. \emph{Verified closure} (Sec.~\ref{sec:progress}) turns the
kernel's binary verdict into a graded fitness read off the proof DAG, so two
unfinished attempts can still be ranked. \emph{Schema recombination}
(Sec.~\ref{sec:recombination}) carries a verified result from one proof into another
as a single kernel-checked step. Within one target, a \emph{proof DAG archive}
stores candidate proofs; across targets, a \emph{schema library} stores proved
results. Schema extraction adds results to the library, and schema recombination
applies them in later DAGs.

ProofEvolve separates neural proposal from symbolic state transition. The policy
$\pol$ proposes edits and the semantic retriever supplies premises and schemas, but
every change to the proof state passes through the symbolic operators defined below.
At evolutionary iteration $t$, let $\mathcal{Q}_t$ be the target queue, $\Map_{T,t}$
the proof DAG archive for target $T$, $\Lib_t$ the verified theorem schema library,
and $\mathcal{H}_t$ the store of rejected proposals and Lean errors keyed by
$(T,D,s)$. The operational state is
\begin{equation}
\label{eq:state}
\mathcal{S}_t=
\left(\mathcal{Q}_t,\{\Map_{T,t}\}_{T\in\mathcal{Q}_t},
\Lib_t,\mathcal{H}_t\right).
\end{equation}
Each $\Map_{T,t}$ is local to one target, whereas $\Lib_t$ persists across targets.
The trusted projection is $\Sigma_t=(\{\Map_{T,t}\}_T,\Lib_t)$, and only
kernel-verified DAG transitions, archive updates with accepted DAGs, and
kernel-checked schema insertions may modify it.

\subsection{Neural Proof Proposal}
\label{sec:operators}

\noindent
At each step the policy $\pol$ proposes an edit $\delta$ to one frontier state
$s\in\mathrm{frontier}(D)$, and the kernel decides whether it survives. The trusted
transition operator is
\begin{equation}
\mathrm{step}_{\K}(D,\delta)=
\begin{cases}
D', & \substack{\text{if Lean accepts the induced edge and}\\
                 \text{$D'$ is an acyclic extension of $D$}},\\
\bot, & \text{otherwise}.
\end{cases}
\label{eq:step}
\end{equation}
An accepted $D'$ extends $D$ by a typed edge carrying the realizer of
Eq.~\eqref{eq:realizer} and stays finite and acyclic, whereas a rejected proposal
leaves the proof DAG archive and schema library unchanged. A proposal takes one of
three forms. In \emph{decomposition}, the model proposes either a checked closing
term or a proof constructor with typed intermediate obligations, and any
proposal-time hole must become an explicit child state before acceptance
\citep{dsp,hilbert}. In \emph{schema recombination}, the model applies a library
schema, whose mechanics we defer to Section~\ref{sec:recombination}. In
\emph{repair}, available only for a previously rejected proposal at the same state,
the model receives the proposal, retrieval context, and Lean error, and the
corrected edge must pass the same $\mathrm{step}_{\K}$ check \citep{hilbert}.

\paragraph{Kernel-grounded selection.}
\label{sec:progress}
Root verification is binary, but an accepted DAG records which internal obligations
are already proved, and we turn this structure into a kernel-grounded fitness
functional. For every nonempty child set, let $w_e(s')>0$ with
$\sum_{s'\in\mathrm{ch}(e)}w_e(s')=1$. Edge and state values are defined together by
well-founded recursion over the acyclic DAG, evaluated from leaves to root in
reverse topological order. For a non-closing edge,
\begin{equation}
\cm_D(e)=
\sum_{s'\in\mathrm{ch}(e)}w_e(s')\cm_D(s').
\label{eq:edge-closure}
\end{equation}
For a state $s$,
\begin{equation}
\label{eq:rho}
\cm_D(s)=
\begin{cases}
1, & \mathrm{Closed}_D(s),\\
0, & \mathrm{out}_D(s)=\varnothing,\\
\displaystyle \max_{e\in\mathrm{out}_D(s)} \cm_D(e),
& \text{otherwise},
\end{cases}
\end{equation}
We use uniform weights and write $\cm(D)=\cm_D(r)$. The maximum encodes alternative
edges at an \textsc{or} node, and the weighted sum encodes the conjunctive
obligations along an \textsc{and} edge. Reverse topological induction gives
$\cm(D)\in[0,1]$. If $r$ is closed, the first case of Eq.~\eqref{eq:rho} gives
$\cm(D)=1$. Conversely, if an open state has value one, some outgoing edge has
weighted average one, and positivity of the weights forces every child to have value
one. Induction then closes every child and hence the source, a contradiction.
Therefore
\begin{equation}
\cm(D)=1\Longleftrightarrow\mathrm{Closed}_D(r).
\label{eq:closure-complete}
\end{equation}
Finally, if $D\preceq D'$, every earlier alternative remains available and closed
nodes stay closed, so $\cm(D')\geq\cm(D)$: verified closure is monotone under
extension and computed entirely from kernel-accepted proof DAGs.

\paragraph{Structural diversity.}
To preserve distinct proof strategies, let $\mathcal{B}$ be a fixed descriptor space
and let $b(D)\in\mathcal{B}$ record binned depth, dominant tactic family, and the
region of the schema index used by $D$. The archive $\Map_T$ \citep{mouret2015mapelites} stores at most one DAG
per descriptor, and an accepted challenger $D'$ updates its cell by
\begin{equation}
\label{eq:map}
\Map_T[b(D')]\leftarrow
\begin{cases}
D', & b(D')\notin\mathrm{dom}(\Map_T),\\
D', & \cm(D')>\cm(\Map_T[b(D')]),\\
\Map_T[b(D')], & \text{otherwise},
\end{cases}
\end{equation}
so the incumbent wins ties. For a temperature $\tau>0$, parents are sampled from
occupied cells by
\begin{equation}
P(D\mid\Map_T)=
\frac{\exp(\cm(D)/\tau)}
{\sum_{D''\in\mathrm{range}(\Map_T)}
\exp(\cm(D'')/\tau)}.
\label{eq:parent}
\end{equation}
The archive therefore retains structural diversity while verified closure supplies
selection pressure.

\paragraph{Frontier scheduling.}
\label{sec:loop}
Algorithm~\ref{alg:proofevolve} (Appendix~\ref{app:alg}) gives one iteration of
ProofEvolve, with parent selection following Eq.~\eqref{eq:parent}. For
$s\in\mathrm{frontier}(D)$, let $\cm_D^{[s\mapsto1]}$ denote the recursion of
Eqs.~\eqref{eq:edge-closure}--\eqref{eq:rho} with the value at $s$ counterfactually
fixed to one. Frontier scheduling then uses
\begin{equation}
\Delta_D(s)=
\cm_D^{[s\mapsto1]}(r)-\cm_D(r),
\label{eq:frontier-gain}
\end{equation}
which prioritizes the frontier state whose closure would yield the largest
structural gain.

\subsection{Symbolic Knowledge Inheritance}
\label{sec:recombination}

\noindent
Recombination is the operator that carries verified work across targets. A closed
state may depend on variables and assumptions from its local context, and
kernel-checked schema extraction turns that local result into a reusable pair
$(\ell,\pi_\ell)$: the schema $\ell$ quantifies the free local variables
$\mathbf{x}$, makes each used hypothesis $A_i$ an explicit premise, and has
conclusion $C$, while $\pi_\ell$ is its proof term:
\begin{equation}
\ell:\forall\mathbf{x},\
A_1\rightarrow\cdots\rightarrow A_m\rightarrow C,
\qquad
\Env\vdash_{\K}\pi_\ell:\ell.
\label{eq:schema}
\end{equation}
For this prenex schema, $\concl(\ell)=C$. Writing $\mathrm{Cl}(D)=\{s\in
V(D):\mathrm{Closed}_D(s)\}$, the states newly closed by an extension $D\preceq D'$
are
\begin{equation}
\mathrm{NewClose}(D,D')=
\mathrm{Cl}(D')\setminus\mathrm{Cl}(D).
\label{eq:new-close}
\end{equation}
For each $s\in\mathrm{NewClose}(D,D')$, the extraction map follows
$\mathrm{win}_{D'}$ and abstracts, in dependency order, exactly the free local
constants and hypotheses occurring in $\mathrm{Asm}_{D'}(s)$. It returns no schema
if this generalization or the displayed kernel judgment fails. We write
$\mathrm{Extract}_{\K}(D,D')$ for the checked schemas obtained from this set.

At an open state $s=(\seq{\Gamma}{g})$, a schema is applicable when a typed
substitution makes its conclusion definitionally equal to the goal. A substitution
$\sigma$ maps the binder telescope $\mathbf{x}$ to Lean terms in $\Gamma$, and we
write $\mathrm{Adm}_{\Gamma}(\sigma)$ when every instantiated binder elaborates in
$\Gamma$ without unresolved metavariables. Then
\begin{equation}
\label{eq:libs}
\mathsf{App}_{\Lib}(s)
=\{(\ell,\sigma): \concl(\ell)\sigma\equiv_{\Env} g
\land\mathrm{Adm}_{\Gamma}(\sigma)\}.
\end{equation}
For $(\ell,\sigma)\in\mathsf{App}_{\Lib}(s)$, Lean first attempts local witnesses
$p_i\in\mathsf{Prf}_{\Env}(\seq{\Gamma}{A_i\sigma})$. Each premise without such a
witness becomes a child state $(\seq{\Gamma}{A_i\sigma})$, and the instantiated
schema application is elaborated as a complete hyperedge realizer before acceptance.
The semantic retriever returns a Mathlib premise shortlist and a library-schema
shortlist,
\begin{equation}
\mathcal{R}_{M}(s)\subseteq\mathsf{Thm}(\Env),
\qquad
\mathcal{R}_{\Lib}(s)
\subseteq\{\ell\mid(\ell,\pi_\ell)\in\Lib\},
\label{eq:schema-shortlist}
\end{equation}
and the candidates presented to the policy are
\begin{equation}
\mathcal{C}_{\Lib}(s)=
\{(\ell,\sigma)\in\mathsf{App}_{\Lib}(s):
\ell\in\mathcal{R}_{\Lib}(s)\}.
\label{eq:recomb-candidates}
\end{equation}
To recombine, the model selects $(\ell,\sigma)\in\mathcal{C}_{\Lib}(s)$, and Lean accepts the result after checking the instantiated schema, all local witnesses, and all residual child obligations as one realizer. Mathlib retrieval separately supplies premises to the proposal context. A persistent schema therefore enters a later proof only as a kernel-checked transformation of its state.

\subsection{Theoretical analysis}
\label{sec:soundness}

We formalize the soundness of ProofEvolve by showing that every proof object
accepted into the proof DAG archive or schema library is validated by the Lean
kernel. The result below states that this property is preserved as the search
extends proof DAGs and grows the library. Let $\mathfrak{D}_0$ contain the
initial DAGs. For $t>0$, $\mathfrak{D}_t$ also contains every challenger
accepted during transitions $0,\ldots,t-1$, including challengers later
discarded by archive comparison.


\begin{restatable}[Invariance of kernel-grounded state]
{theorem}{soundnessTheorem}
\label{thm:soundness}
Under the formal assumptions in Appendix~\ref{app:safety}, every finite execution
$\mathcal{S}_0\rightarrow\cdots\rightarrow\mathcal{S}_n$ of
Algorithm~\ref{alg:proofevolve} satisfies, for each $0\leq t\leq n$:
\begin{list}{}{\setlength{\leftmargin}{2.4em}\setlength{\labelwidth}{2.0em}\setlength{\labelsep}{0.4em}\setlength{\itemsep}{2pt}\setlength{\topsep}{3pt}\setlength{\parsep}{0pt}}
    \item[\textup{(I1)}] every accepted edge in every $D\in\mathfrak{D}_t$
    has a checked realizer of the form in Eq.~\eqref{eq:realizer};
    \item[\textup{(I2)}] every closed node $s$ in every $D\in\mathfrak{D}_t$
    satisfies $\mathrm{Asm}_D(s)\in\mathsf{Prf}_{\Env}(s)$; and
    \item[\textup{(I3)}] every $(\ell,\pi_\ell)\in\Lib_t$ satisfies
    $\Env\vdash_{\K}\pi_\ell:\ell$.
\end{list}
For every $0\leq t<n$, $\Lib_t\subseteq\Lib_{t+1}$. If transition $t$ rejects
its proposal, then $\Sigma_{t+1}=\Sigma_t$.
\end{restatable}

Building on Theorem~\ref{thm:soundness}, we obtain the validity of the proofs
ProofEvolve returns.
\begin{restatable}[Validity of returned proofs]
{corollary}{returnedProofCorollary}
\label{cor:returned-proof}
Under the assumptions of Theorem~\ref{thm:soundness}, if ProofEvolve returns
$p$ from $D'\in\mathfrak{D}_n$ for a target $T$ with root
$r=(\seq{\Gamma_0}{T})$, then
$\Env;\Gamma_0\vdash_{\K}p:T$.
\end{restatable}

The neural models decide only which variations to attempt, whereas the kernel
decides whether an accepted one is valid. Every proof ProofEvolve returns therefore
type-checks against the standard axioms by construction. The assumptions and full
proofs of Theorem~\ref{thm:soundness} and Corollary~\ref{cor:returned-proof} are deferred to Appendix~\ref{app:safety}.

\newcommand{\micon}[1]{%
  \raisebox{-0.22em}{\includegraphics[height=1.05em]{Figures/#1}}\,%
}

\begin{table}[t]
\centering
\setlength{\tabcolsep}{4pt}
\vskip -10pt
\begin{tabular}{@{}lcccc@{}}
\toprule
Method & Putnam (\%) & IMO-Lean (\%) & Combi (\%) & Avg.\ (\%) \\
\midrule

\rowcolor{gray!10}
\multicolumn{5}{c}{\itshape Inference-only models} \\
\micon{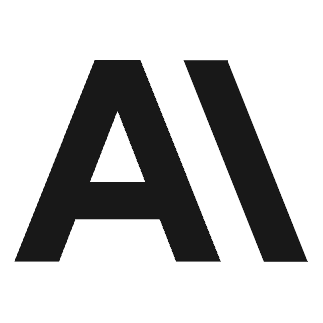}Claude Haiku 4.5
\citep{anthropic2025claudehaiku45}
& $0.0$ & $0.0$ & $3.3$ & $1.1$ \\
\micon{anthropic_ic.png}Claude Sonnet 4.6
\citep{anthropic2026claudesonnet46}
& $0.0$ & $0.0$ & $6.7$ & $2.2$ \\
\micon{anthropic_ic.png}Claude Opus 4.8
\citep{anthropic2026claudeopus48}
& $0.0$ & $0.0$ & $10.0$ & $3.3$ \\
\micon{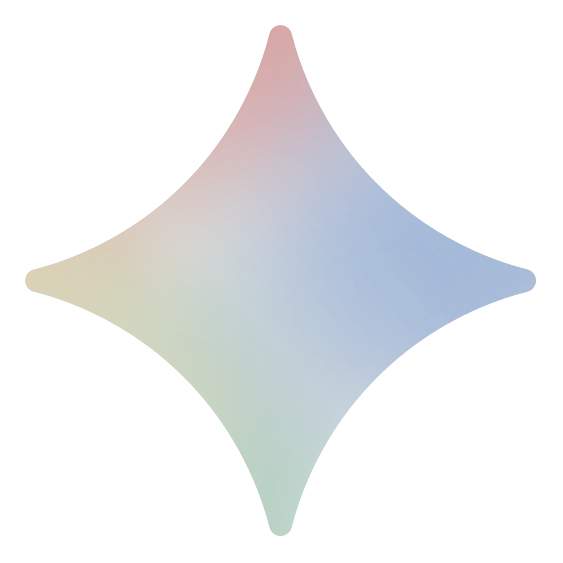}Gemini 3.1 Pro
\citep{googledeepmind2026gemini31pro}
& $0.0$ & $3.3$ & $10.0$ & $4.4$ \\
\micon{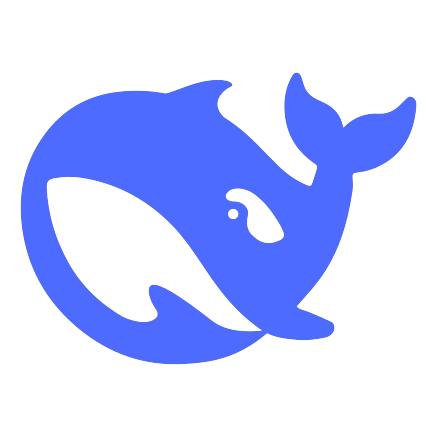}DeepSeek-Prover-V2-671B
\citep{ren2025deepseekproverv2}
& $7.0$ & $0.0$ & $10.0$ & $5.7$ \\
Goedel-Prover-V2-32B
\citep{route2025goedelproverv2}
& $12.8$ & $5.0$ & $0.0$ & $5.9$ \\
\micon{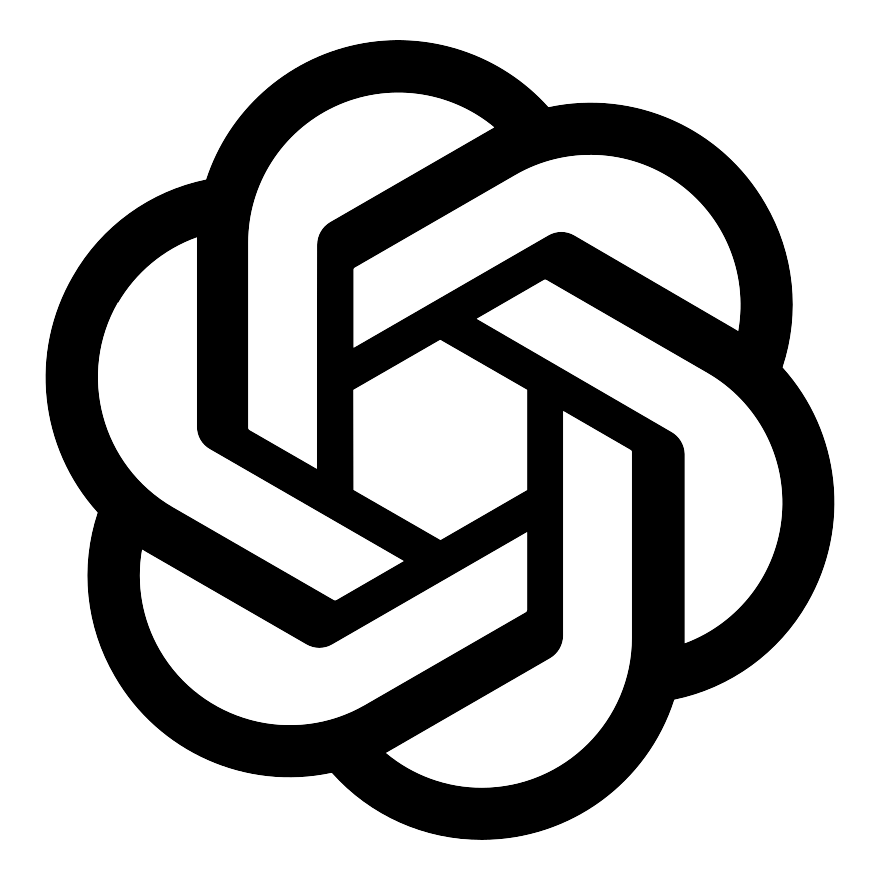}GPT-5.5
\citep{openai2026gpt55}
& $10.0$ & $5.0$ & $13.0$ & $9.3$ \\

\midrule
\rowcolor{gray!10}
\multicolumn{5}{c}{\itshape Agentic systems} \\
ReAct~\citep{yao2023react} (Claude Opus 4.8)
& $35.0$ & $15.0$ & $27.0$ & $25.7$ \\
Aristotle \citep{aristotle} (Claude Opus 4.8)
& $45.0$ & $13.3$ & $40.0$ & $32.8$ \\
AxProver \citep{axprover} (Claude Opus 4.8)
& $54.3$ & $10.0$ & $47.0$ & $37.1$ \\
Hilbert \citep{hilbert} (Claude Opus 4.8)
& $55.5$ & $33.3$ & $49.0$ & $45.9$ \\
LEAP \citep{leap} (Claude Opus 4.8)
& $64.7$ & $36.7$ & $\mathbf{50.0}$ & $50.5$ \\

\midrule
\rowcolor{gray!10}
\multicolumn{5}{c}{\itshape Our work} \\
\textbf{ProofEvolve} (Claude Opus 4.8)
& $\mathbf{71.2}$ & $\mathbf{53.3}$ & $49.0$ & $\mathbf{57.8}$ \\
\bottomrule
\end{tabular}

\caption{\textbf{Main comparison.} Mean solve rate (\%) over three independent runs on PutnamBench
\citep{tsoukalas2024putnambench}, IMO-LeanProofBench
\citep{luong2025imobench}, and CombiBench
\citep{liu2025combibench}. }
\vskip -10pt
\label{tab:main}
\end{table}

\section{Experiments}
\label{sec:exp}

\subsection{Experimental setup}

\paragraph{Implementation.}
In our experiments, the parameters of all base LLMs remain frozen throughout search and across targets. For all agentic baselines, we use Claude Opus 4.8 as base model to ensure fair comparison. Lean~4 with Mathlib provides tactic states, elaboration errors, and kernel verification. Each target receives a budget of parallel attempts together with a bounded kernel-guided repair loop, and ProofEvolve draws its attempts from this budget. All experiments use Lean~4 with the same Mathlib version. For proprietary models, we use API calls. For open-sourced model hosting, we use NVIDIA B200 GPU clusters with 192 GB of memory per GPU. Each node contains 8 GPUs, and our largest runs use up to 28 nodes, corresponding to 224 GPUs operating concurrently.

\paragraph{Benchmarks and baselines.}
We evaluate on three competition-level benchmarks. PutnamBench \citep{tsoukalas2024putnambench} formalizes problems from the William Lowell Putnam Mathematical Competition. We use its pure-proof subset, which excludes problems whose theorem statements already contain a fixed answer value. IMO-LeanProofBench \citep{luong2025imobench} contains Lean formalizations of International Mathematical Olympiad-level proof problems. CombiBench \citep{liu2025combibench} covers competition-level combinatorial mathematics, where a proof usually rests on an explicit construction or count. We compare against pass@$16$ sampling from the LLMs and five agentic systems: LEAP, Hilbert, AxProver, Aristotle, and ReAct. Our reproduction of each agentic system uses the same base LLM, retains its search strategy, and runs under a matched budget.

\paragraph{Evaluation metrics.}
Our primary metric is the solve rate, the fraction of a benchmark whose theorems are proved and pass the Lean~4 kernel verification. A theorem is evaluated as solved only when its final proof matches the benchmark ground truth, contains no unresolved metavariables or placeholders, and passes the Lean kernel. We exclude proofs that rely on \texttt{native\_decide}, because its code-generation path introduces an axiom outside the standard proof kernel. We apply the same verification to every solution in our evaluation. Every reported solve is independently re-verified against the matched Lean kernel with a restricted \texttt{\#print axioms} check. Across more than $400$ re-verifications, we found $0$ false positives.


\subsection{Main results}
\label{sec:main}

Table~\ref{tab:main} reports solve rates on the three benchmarks. Under pass@$16$ sampling, Claude Opus 4.8 without agentic search solves $0.0\%$ of PutnamBench and IMO-LeanProofBench. The agentic systems use the same base model Claude Opus 4.8, providing the closest matched comparison of their search methods. ProofEvolve has the highest average solve rate at $57.8\%$, ahead of LEAP ($50.5\%$) and Hilbert ($45.9\%$). It leads PutnamBench at $71.2\%$, $6.5$ points above LEAP, and widens the margin on IMO-LeanProofBench, reaching $53.3\%$ against $36.7\%$ for LEAP. On CombiBench the strongest systems are within one point, LEAP at $50.0\%$ and ProofEvolve at $49.0\%$. The largest margin appears on IMO-LeanProofBench, whose problems often require proofs assembled from several lemmas. This pattern is consistent with the intended role of graded selection and schema reuse on decomposable problems.

\subsection{Dynamics of verified closure \texorpdfstring{$\cm$}{rho}}
\label{sec:analysis}

We study how verified closure $\cm$ changes during proof search. ProofEvolve uses $\cm$ in Eq.~\eqref{eq:rho} as a fitness value computed from kernel-accepted edges. Unlike a binary root verdict, it records partial progress once Lean certifies intermediate subgoals. Since Eq.~\eqref{eq:closure-complete} guarantees that solved runs reach $\cm=1$, we focus on the search trajectory before completion. Figure~\ref{fig:rho} shows that $\cm$ increases step by step as subgoals are verified, while failed runs plateau below one. Figure~\ref{fig:dag} shows one solved run in which several lemmas are certified before the root is finally closed. These results show that $\cm$ captures verified intermediate progress that binary feedback cannot represent.

\begin{figure}[t]
\centering
\begin{subfigure}[b]{0.372\textwidth}
  \centering
  \includegraphics[width=\linewidth]{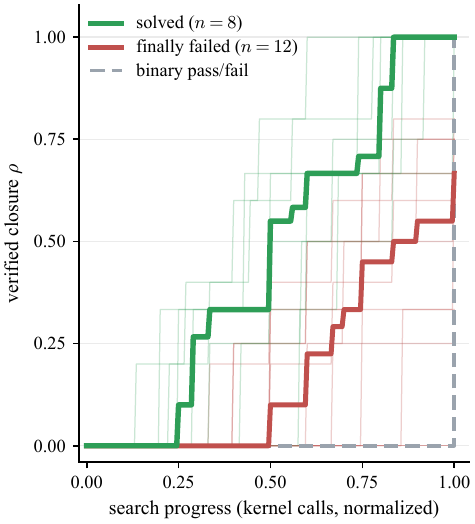}
  \caption{Closure trajectories}
  \label{fig:rho}
\end{subfigure}\hfill
\begin{subfigure}[b]{0.598\textwidth}
  \centering
  \includegraphics[width=\linewidth]{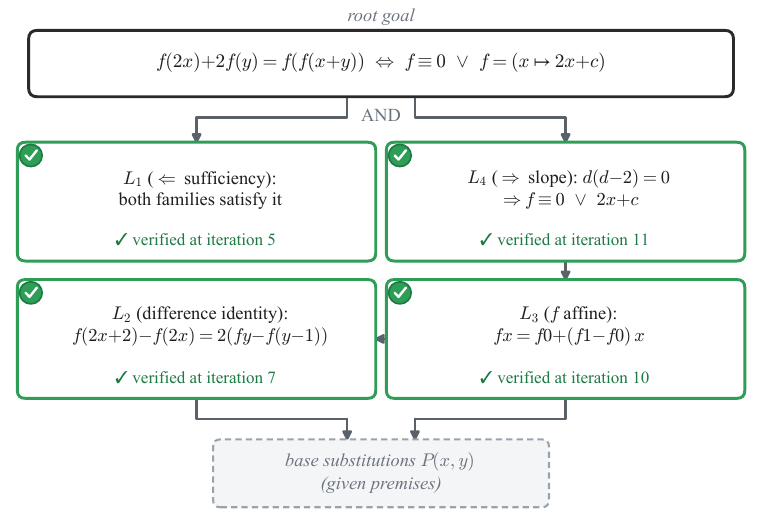}
  \caption{One solved proof DAG}
  \label{fig:dag}
\end{subfigure}
\vskip -2pt
\caption{\textbf{Verified closure $\cm$ during search.}
\textbf{(a)}~$\cm$ against evolutionary iterations: solved runs (green) reach $1$, failed
runs (red) plateau below, and the binary pass/fail signal (dashed) stays at $0$
(medians with interquartile bands).
\textbf{(b)}~An accepted proof DAG whose lemma nodes are certified by the kernel at
iterations $5,7,10,11$, so $\cm$ rises step by step to $1$.}
\vskip -10pt
\label{fig:closure}
\end{figure}



\begin{figure}[t]
\centering
\includegraphics[width=0.44\textwidth]{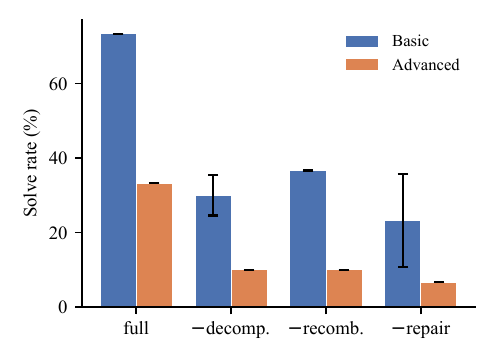}
\vskip -4pt
\captionsetup{width=0.68\textwidth}
\caption{\textbf{Per-difficulty ablation:} every variation operator contributes more on the harder Advanced split than on the Basic split. Error bars show the run-to-run standard deviation across independent reruns.}
\vskip -8pt
\label{fig:difficulty}
\end{figure}

\subsection{Ablation study}
\label{sec:ablation}

We conduct ablation study on how each of the three variation operators affects ProofEvolve. Decomposition breaks a goal into smaller subgoals. Repair fixes a failed step using the error message from Lean~4. Recombination reuses an already proved result to close a new goal. We remove one operator at a time on the $60$ IMO-LeanProofBench problems. The base model and the compute budget stay the same, and we repeat each run with five random seeds to ensure statistical stability. As shown in Figure~\ref{fig:difficulty}, the full system on average solves $32$ of the $60$ problems, with $22$ of $30$ on the Basic split and $10$ of $30$ on the Advanced split. Without decomposition it on average solves $11$ ($9$ Basic, $2$ Advanced), without recombination $14$ ($11$ Basic, $3$ Advanced), and without repair $9$ ($7$ Basic, $2$ Advanced). The drop is larger on the harder Advanced split than on the Basic split. These results show that all three variation operators contribute to the performance of ProofEvolve, with larger effects on the Advanced split.

\begin{figure}[t]
\centering
\includegraphics[width=\textwidth]{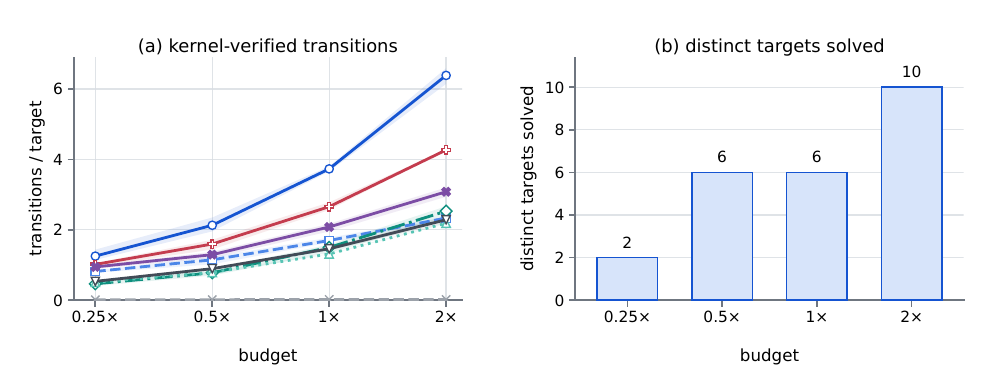}
\vskip -6pt
\caption{\textbf{Test-time budget scaling} (a) Kernel-verified transitions per target across seeds; (b) Union of distinct targets solved across seeds and configuration. \emph{In (a)}: \micon{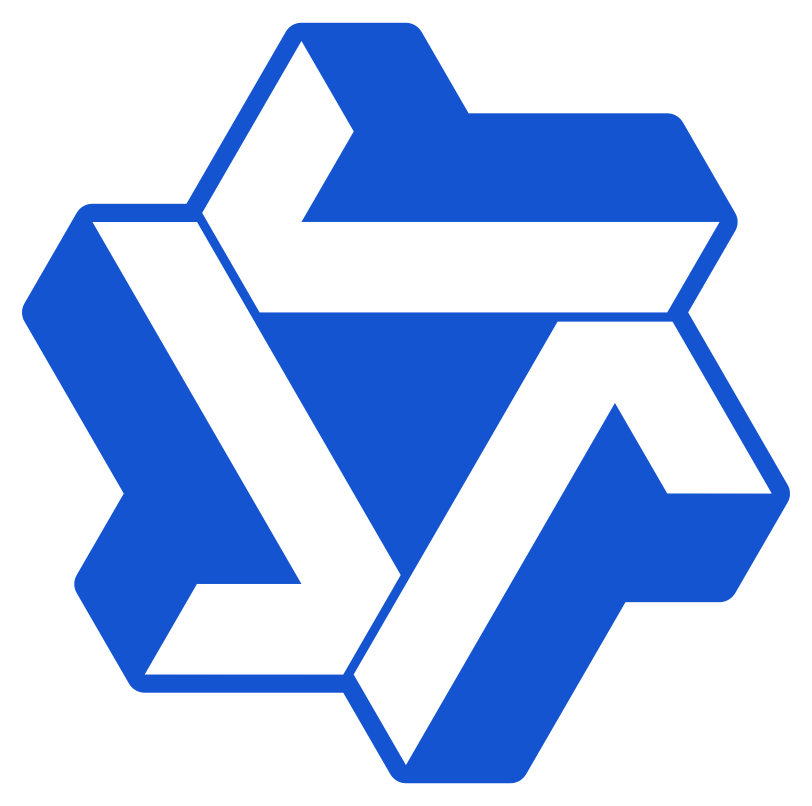}Qwen3.5 \textcolor[HTML]{1554D1}{think~$\circ$}/\textcolor[HTML]{4C86E8}{instant~$\square$}; \ \micon{qwen_mark.png}Qwen3.6 \textcolor[HTML]{0E8F7E}{instant~$\Diamond$}/\textcolor[HTML]{57C4B4}{think~$\triangle$}; \ \  \micon{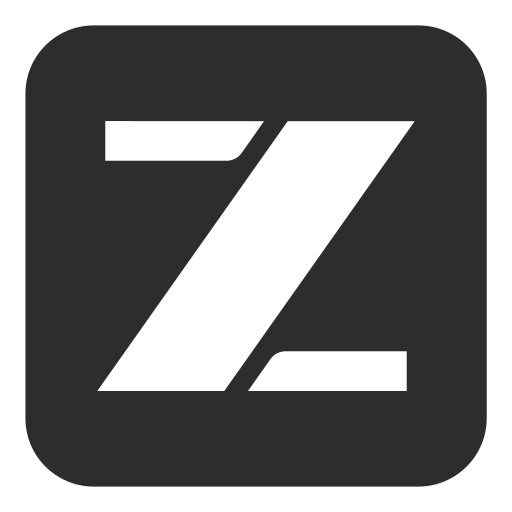}GLM-5.1 \textcolor[HTML]{C43C4E}{$+$};  \ \ \micon{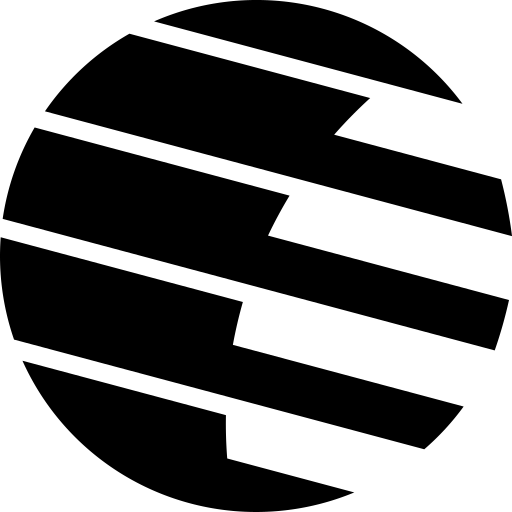}Kimi-K2.6 \textcolor[HTML]{7C4DA5}{$\times$}; \ {\micon{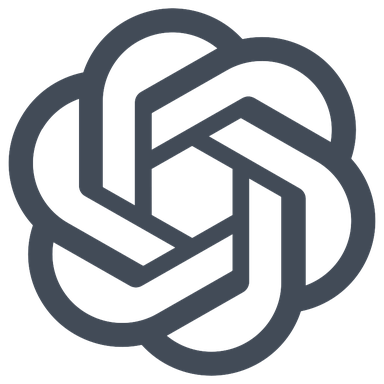}gpt-oss \textcolor[HTML]{424B57}{med~$\triangledown$}/\textcolor[HTML]{9CA2AA}{low~$\times$}}.}
  \vskip -10pt
\label{fig:openweight}
\end{figure}

\subsection{Test-time budget scaling}
\label{sec:openweight}
 We next test ProofEvolve with open-weight models as the per-target budget grows. The models are Qwen3.5-397B~\citep{qwen2026qwen35fp8card}, Qwen3.6-35B~\citep{qwen2026qwen36card}, Kimi-K2.6~\citep{moonshot2026kimik26card}, GLM-5.1~\citep{zai2026glm51fp8card}, and gpt-oss-120b~\citep{openai2026gptoss120bcard}, each run in the inference modes it supports, for eight configurations. A per-target budget caps the model calls, Lean-kernel calls, tokens, and wall-clock time each target may use. Appendix~\ref{app:openweight} gives the models, decoding, and the full budget table. Seven of the eight configurations produce more kernel-verified transitions per target as the budget grows (Figure~\ref{fig:openweight}a, per-seed values in Appendix~\ref{app:openweight-grid}). The one exception, gpt-oss-120b at low reasoning effort, is flat across budgets, confirming that the trend is not merely a by-product of issuing more calls. The final results  in Figure~\ref{fig:openweight}b shows that the number of distinct targets solved increases monotonically.  Appendix~\ref{app:openweight-grid} reports the exact per-budget counts. We also provide the solutions found by the models in Appendix~\ref{app:openweight-cases}.

\pgfplotsset{heldoutaxis/.style={
  width=\linewidth, height=4.2cm,
  ymin=-1, ymax=6.6, ytick={0,2,4,6},
  ylabel={Lift over zero-shot (pp)},
  ylabel near ticks, xlabel near ticks,
  label style={font=\footnotesize}, tick label style={font=\scriptsize},
  grid=major}}

\begin{figure}[t]
\centering
\begin{subfigure}[t]{0.485\textwidth}
\centering
\begin{tikzpicture}
\begin{axis}[heldoutaxis,
  xlabel={Library size (verified self-solutions)},
  xtick={0,2000,4000,5546}, xticklabels={$0$,$2$k,$4$k,full}]
\addplot[name path=upper,draw=none,forget plot] coordinates
  {(0,0)(1000,4.08)(2000,4.20)(4000,5.01)(5546,4.49)};
\addplot[name path=lower,draw=none,forget plot] coordinates
  {(0,0)(1000,2.56)(2000,1.08)(4000,3.33)(5546,3.31)};
\addplot[blue!15,forget plot] fill between[of=upper and lower];
\addplot[blue,mark=*,thick,mark size=1.6pt] coordinates
  {(0,0)(1000,3.32)(2000,2.64)(4000,4.17)(5546,3.90)};
\addplot[black,dashed,thick] coordinates {(0,0)(5546,0)};
\addplot[red,only marks,mark=triangle*,mark size=2.6pt] coordinates {(5546,0.13)};
\end{axis}
\end{tikzpicture}
\caption{Library size, at $K{=}8$}
\label{fig:heldout-size}
\end{subfigure}\hfill
\begin{subfigure}[t]{0.485\textwidth}
\centering
\begin{tikzpicture}
\begin{axis}[heldoutaxis,
  xlabel={Retrieved examples $K$}, xtick={0,8,16,32,64}]
\addplot[name path=upper,draw=none,forget plot] coordinates
  {(0,0)(8,4.49)(16,5.19)(32,5.03)(64,6.02)};
\addplot[name path=lower,draw=none,forget plot] coordinates
  {(0,0)(8,3.31)(16,3.95)(32,1.79)(64,4.64)};
\addplot[blue!15,forget plot] fill between[of=upper and lower];
\addplot[blue,mark=*,thick,mark size=1.6pt] coordinates
  {(0,0)(8,3.90)(16,4.57)(32,3.41)(64,5.33)};
\addplot[black,dashed,thick] coordinates {(0,0)(64,0)};
\end{axis}
\end{tikzpicture}
\caption{Retrieval depth, full library}
\label{fig:heldout-depth}
\end{subfigure}
\vskip -2pt
\caption{\textbf{Reusing the prover's own verified proofs.} Solve-rate lift over
zero-shot, in percentage points, on the $744$ screened evaluation theorems
(three-run mean). \textcolor{blue}{$\bullet$}~relevant retrieval;
\textcolor{black}{\textbf{--\,--}}~zero-shot;
\textcolor{red}{$\blacktriangle$}~random retrieval from the same library. The band is
$\pm1$ run-to-run standard deviation \emph{of the lift}, which is not the standard
deviation of the solve rate reported in Table~\ref{tab:heldout}. The lift is $0$ at
an empty library and at $K{=}0$, where both conditions reduce to zero-shot.}
\vskip -10pt
\label{fig:heldout}
\end{figure}
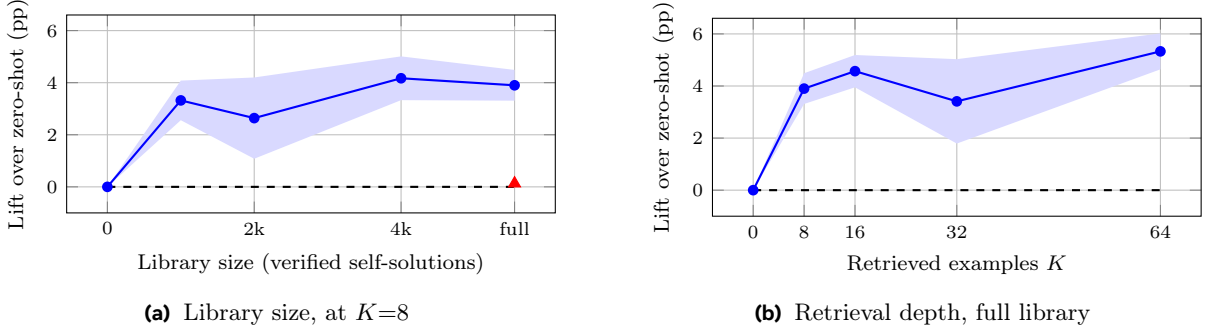

\subsection{Old Proofs, New Theorems: Verified Reuse Across Problems}
\label{sec:transfer}

A library that grows during evaluation should make later targets easier to prove. We
test this at two scales. The controlled study below isolates accumulation in a setting
where the dependency structure is known by construction.

\paragraph{Controlled compositional families.}
To isolate the effect of a library that grows during evaluation, we run a controlled study on synthetic compositional lemma families, where each later target is built from lemmas that earlier targets establish. ProofEvolve solves $19.8\%$ of the targets when the library grows across targets and $7.3\%$ when the library is reset before each target. The growing library therefore solves $2.7\times$ as many targets, with every other component held fixed. The study isolates the inheritance mechanism in this controlled setting, separate from the benchmark evaluation. Recombination stays sound throughout, because a typed schema discharges a subgoal only when Lean accepts its instantiation, so a mismatched retrieval fails without changing the trusted proof state.

\paragraph{Unseen Lean Workbook theorems.}
The compositional families are synthetic. We next evaluate cross-problem reuse on
real competition-style theorems, using a library built entirely from the prover's
own work. We start from Lean~Workbook
\citep{ying2024leanworkbook} statements with the machine-generated proofs released by
Goedel-Prover \citep{lin2025goedelproverv1}. Re-verification under the same kernel and
axiom checks retains $20{,}554$ statement and proof pairs, and deduplication leaves
$10{,}968$ distinct theorems. We hold out $1{,}000$ for evaluation and use the
remaining $9{,}968$ as a source stream. The base model, Qwen3.5-397B-A17B-FP8
\citep{qwen2026qwen35fp8card}, attempts every buildable source theorem once; its
$5{,}546$ kernel-accepted proofs form the library. Five independent language-model
judges screen the evaluation theorems against their retrieved neighbors, leaving the
$744$ theorems that fewer than two judges flag. For each evaluation theorem the
semantic retriever selects the top-$K$ schemas, and their statements and Lean-accepted
proof bodies enter the proposal context. \emph{Zero-shot} omits that context, and
\emph{random retrieval} supplies $K$ schemas drawn uniformly from the same library.
All conditions share one prompt template, one Lean environment and one decoding
profile. We report the mean over three runs, with condition contrasts computed as
paired differences across matched runs.

At $K{=}8$ the library raises the solve rate from $49.5\%$ to $53.4\%$, a gain of
$3.9$ points, while random retrieval from the same library reaches $49.6\%$: the
improvement comes from selecting useful verified work, not from adding examples to the
prompt. The gain appears at every library size and retrieval depth we measure. A
library of only $1{,}000$ proofs already adds $3.3$ points, and the gain reaches $5.3$
points at $K{=}64$. Figure~\ref{fig:heldout} shows both scaling axes and
Table~\ref{tab:heldout} lists every condition. Across the three runs at $K{=}8$,
relevant retrieval closes $351$ theorem instances that zero-shot leaves open, and in
$322$ of them, or $91.7\%$, the accepted proof does not reproduce any shown proof
verbatim. The library supplies reusable proof structure rather than a catalog of
answers.

This study isolates one mechanism. Each condition gives a single whole-proof attempt
with no repair and no second sample, so retrieved schemas act as in-context exemplars
rather than as typed instantiations composed into a realizer
(Eqs.~\eqref{eq:libs}--\eqref{eq:recomb-candidates}); the DAG archive, decomposition,
repair, and verified closure as a selection signal are all switched off. Holding the
search fixed is what makes the library's own contribution measurable.
Appendix~\ref{app:heldout} records the full setup and Appendix~\ref{app:examples}
shows two retrieval-to-proof traces.



\section{Conclusion}

We introduced ProofEvolve, a neuro-symbolic evolutionary framework that improves through explicit, formally verified proof structures. It represents partial proofs as AND-OR DAGs, ranks them using kernel-grounded verified closure, preserves structurally diverse candidates, and extracts closed sub-DAGs as reusable theorem schemas. The symbolic Lean~4 kernel faithfully checks every proposed variation before it can evolve the internal structured knowledge. Empirically, ProofEvolve achieves the highest average solve rate among the state-of-the-art baselines on three challenging competition-level Lean benchmarks, and on Lean~Workbook theorems disjoint from its library a library of the prover's own verified proofs adds about four points over zero-shot, with random retrieval from the same library adding nothing. More broadly, this work paves a concrete step toward recursively self-improving agents that accumulate formal knowledge over time. This insight could shed light on scientific discovery in other scientific domains, such as theoretical physics and chemistry. We hope this work can inspire future research on the field of continual learning for AI-driven scientific discovery.

\bibliographystyle{assets/plainnat}
\bibliography{paper}
\clearpage
\appendix
\section*{\LARGE Appendix}
\section{Algorithm}
\label{app:alg}

The full evolutionary iteration summarized in Section~\ref{sec:loop} is given below.

\begin{algorithm}[h]
\caption{One evolutionary iteration in ProofEvolve}
\label{alg:proofevolve}
\textbf{Input}: target queue $\mathcal{Q}$; policy $\pol$; kernel $\K$;
archives $\{\Map_T\}$; schema library $\Lib$; error store $\mathcal H$
\begin{algorithmic}[1]
\STATE Select $T\in\mathcal{Q}$. If $\Map_T$ is empty, insert the singleton root DAG.
\STATE Sample parent $D$ from $\Map_T$ using Eq.~\eqref{eq:parent}.
\STATE Select $s\in\mathrm{frontier}(D)$ maximizing $\Delta_D(s)$.
\STATE Retrieve Mathlib premises $\mathcal{R}_{M}(s)$ and schemas
$\mathcal{R}_{\Lib}(s)$; form $\mathcal{C}_{\Lib}(s)$.
\IF{$\mathcal H(T,D,s)$ contains a rejected edit and error $\epsilon$}
    \STATE Sample $\delta\sim\pol(\textsc{repair}\mid s,D,\mathcal{R}_{M}(s),
    \mathcal{C}_{\Lib}(s),\epsilon)$.
\ELSE
    \STATE Choose an applicable
    $o\in\{\textsc{decompose},\textsc{recombine}\}$.
    \STATE Sample $\delta\sim\pol(o\mid s,D,\mathcal{R}_{M}(s),
    \mathcal{C}_{\Lib}(s))$.
\ENDIF
\STATE Compute $D'\leftarrow\mathrm{step}_{\K}(D,\delta)$.
\IF{$D'=\bot$}
    \STATE Store the rejected edit and Lean error in $\mathcal H(T,D,s)$;
    \textbf{return} with $\Sigma$ unchanged.
\ENDIF
\STATE $\Lib\leftarrow\Lib\cup\mathrm{Extract}_{\K}(D,D')$.
\STATE Update $\Map_T[b(D')]$ using Eq.~\eqref{eq:map}.
\IF{$\cm(D')=1$ and
$\Env;\Gamma_0\vdash_{\K}\mathrm{Asm}_{D'}(r):T$}
    \STATE \textbf{return} $\mathrm{Asm}_{D'}(r)$.
\ENDIF
\end{algorithmic}
\end{algorithm}

\section{Proofs of Theoretical Results}
\label{app:safety}

\noindent
This appendix states the execution assumptions and derives the invariant and
returned-proof results used in Section~\ref{sec:soundness}.

\subsection{Assumptions}
Theorem~\ref{thm:soundness} uses the following assumptions.
\begin{list}{}{\setlength{\leftmargin}{2.4em}\setlength{\labelwidth}{2.0em}\setlength{\labelsep}{0.4em}\setlength{\itemsep}{2pt}\setlength{\topsep}{3pt}\setlength{\parsep}{0pt}}
    \item[\textup{(A1)}] The environment $\Env$ and kernel $\K$ are fixed throughout
    the execution. Every target that occurs in the execution is well formed in
    $\Env$. Every initial DAG is finite, acyclic, rooted at
    $(\seq{\Gamma_0}{T})$, and each accepted edge satisfies
    Eq.~\eqref{eq:realizer}.
    \item[\textup{(A2)}] Every initial library entry
    $(\ell,\pi_\ell)\in\Lib_0$ satisfies
    $\Env\vdash_{\K}\pi_\ell:\ell$.
    \item[\textup{(A3)}] Every later DAG is produced by
$\mathrm{step}_{\K}$ in Eq.~\eqref{eq:step}. Every later library
entry is produced by $\mathrm{Extract}_{\K}$, which returns a pair
$(\ell,\pi_\ell)$ only if Lean elaborates the resulting declaration
without unresolved metavariables and verifies
$\Env\vdash_{\K}\pi_\ell:\ell$. If abstraction, elaboration, or
kernel verification fails, no schema is returned. Archive and
library updates follow Algorithm~\ref{alg:proofevolve}.
\end{list}

\subsection{Kernel-grounded invariance}
Let $\mathfrak{D}_0$ contain the initial DAGs; for $t>0$, $\mathfrak{D}_t$ also contains every challenger accepted during transitions $0,\ldots,t-1$, including challengers later discarded by archive comparison.
\soundnessTheorem*
\begin{proof}
We use induction on the transition index $t$. Assumption (A1) establishes (I1)
at $t=0$, and (A2) establishes (I3).

To derive (I2), fix $D\in\mathfrak{D}_0$ and a closed node $s$. Write
$\mathrm{win}_D(s)=(s;s_1,\ldots,s_k)$ and define
\begin{equation}
h_D(s)=
\begin{cases}
0, & k=0,\\
1+\displaystyle\max_{1\leq i\leq k}h_D(s_i), & k>0.
\end{cases}
\label{eq:closure-height}
\end{equation}
Acyclicity makes $h_D(s)$ finite. If $h_D(s)=0$, then (I1) gives
\begin{equation}
\mathrm{Asm}_D(s)
=F_{\mathrm{win}_D(s)}(\star)
\in\mathsf{Prf}_{\Env}(s).
\label{eq:assembly-base}
\end{equation}
For $h_D(s)>0$, every $s_i$ has smaller height. The inner induction gives
$p_i=\mathrm{Asm}_D(s_i)\in\mathsf{Prf}_{\Env}(s_i)$, so
\begin{equation}
\mathrm{Asm}_D(s)
=F_{\mathrm{win}_D(s)}(p_1,\ldots,p_k)
\in\mathsf{Prf}_{\Env}(s).
\label{eq:assembly-step}
\end{equation}
This proves (I2) at $t=0$.

Assume (I1)--(I3) at index $t$. If
$\mathrm{step}_{\K}(D,\delta)=\bot$, Algorithm~\ref{alg:proofevolve} updates
only the error store. Therefore
\begin{equation}
\Sigma_{t+1}=\Sigma_t,
\qquad
\Lib_{t+1}=\Lib_t,
\qquad
\mathfrak{D}_{t+1}=\mathfrak{D}_t.
\label{eq:reject-invariant}
\end{equation}
All three invariants follow immediately.

Suppose instead that $\mathrm{step}_{\K}(D,\delta)=D'$. By
Eq.~\eqref{eq:step}, $D\preceq D'$ and the new edge has a checked realizer.
All old edges retain their realizers, so (I1) holds for
\begin{equation}
\mathfrak{D}_{t+1}=\mathfrak{D}_t\cup\{D'\}.
\label{eq:accepted-dags}
\end{equation}
Apply the height induction in Eqs.~\eqref{eq:closure-height}--\eqref{eq:assembly-step}
to every closed node of $D'$. Its witnessing edges are either old edges, covered
by the outer induction hypothesis, or the new edge, covered by
Eq.~\eqref{eq:step}. Hence (I2) holds for $D'$ and remains true for every DAG in
$\mathfrak{D}_t$.

The library update is
\begin{equation}
\Lib_{t+1}
=\Lib_t\cup\mathrm{Extract}_{\K}(D,D')
\supseteq\Lib_t.
\label{eq:library-monotonicity}
\end{equation}
Every extracted pair satisfies Eq.~\eqref{eq:schema}, so (I3) is preserved.
Equation~\eqref{eq:map} stores either $D'$ or the previous valid incumbent.
This completes the outer induction.
\end{proof}

\subsection{Validity of returned proofs}
\returnedProofCorollary*
\begin{proof}
The return guard in Algorithm~\ref{alg:proofevolve}, the equivalence in
Eq.~\eqref{eq:closure-complete}, and Theorem~\ref{thm:soundness}(I2) give
\begin{equation}
\cm(D')=1
\Longrightarrow \mathrm{Closed}_{D'}(r)
\Longrightarrow
\mathrm{Asm}_{D'}(r)\in\mathsf{Prf}_{\Env}(r).
\label{eq:return-derivation}
\end{equation}
The algorithm returns $p=\mathrm{Asm}_{D'}(r)$. By
Eq.~\eqref{eq:proof-set},
\begin{equation}
p\in\mathsf{Prf}_{\Env}(\seq{\Gamma_0}{T})
\Longrightarrow
\Env;\Gamma_0\vdash_{\K}p:T.
\end{equation}
\end{proof}

\section{Setup of Open-weight Models}
\label{app:openweight}
This appendix gives the setup for the test-time compute scaling study in Section~\ref{sec:openweight}. The study uses the ProofEvolve mechanism from Section~\ref{sec:method} and varies the base model and composite per-target budget. Each reported proof elaborates in the frozen Lean~4 environment with the matched Mathlib commit. It contains no unresolved metavariables or placeholders, does not use \texttt{native\_decide}, and passes an independent \texttt{\#print axioms} check. We use seeds $\{19,36,65\}$.

\paragraph{Target set.}
The scheduled evaluation manifest contains $485$ targets from the benchmarks in Table~\ref{tab:main}: PutnamBench ($326$), IMO-LeanProofBench ($60$), and CombiBench ($99$). Each target is scheduled with all three seeds.

\paragraph{Models and serving.}
All open-weight serving ran on the computation nodes with $8\times$ NVIDIA B200 GPUs ($192$\,GB HBM each), dual-socket Intel Xeon
hosts ($224$ vCPUs, about $3.9$\,TB RAM), running Ubuntu~22.04.5 LTS
(kernel \texttt{6.8.0-1040-gcp}) with CUDA~12.8. We serve under Python~3.11.15 with two engines: {vLLM 0.19.0} (PyTorch 2.10.0+cu128) for the Qwen models and
{SGLang 0.5.10} (PyTorch 2.9.1+cu128) for GLM-5.1, Kimi-K2.6, and gpt-oss-120b. Per-model tensor parallelism is listed in Table~\ref{tab:openweight-serving}. Proof
verification uses {Lean~4.29.1} with {Mathlib} commit \texttt{5e932f97}
and {pantograph~0.3.15}. Proprietary baselines (Claude Opus~4.8) are accessed
through the vendor API. The headline results in Table~\ref{tab:main} therefore use API inference plus local Lean kernel verification and do not consume the B200 cluster.

\paragraph{Decoding.}
Table~\ref{tab:openweight-serving} gives the decoding parameters, which remain fixed across budget scales. The per-call output cap is $32{,}768$ tokens for every run.

\begin{table}[t]
\centering\small
\setlength{\tabcolsep}{4.2pt}
\begin{tabular}{@{}llcclccc@{}}
\toprule
Model & Developer & Engine & TP & Mode & temp.\ & top-$p$ & top-$k$ \\
\midrule
\micon{moonshot_mark.png}Kimi-K2.6 & Moonshot AI & SGLang & $8$ & instant & $0.6$ & $0.95$ & -- \\
\micon{qwen_mark.png}Qwen3.5-397B-A17B-FP8 & Alibaba & vLLM & $4$ & instant & $0.7$ & $0.80$ & $20$ \\
 & & & & thinking & $0.6$ & $0.95$ & $20$ \\
\micon{qwen_mark.png}Qwen3.6-35B-A3B & Alibaba & vLLM & $1$ & instant & $0.7$ & $0.80$ & $20$ \\
 & & & & thinking & $1.0$ & $0.95$ & $20$ \\
\micon{glm_mark.png}GLM-5.1-FP8 & Z.ai & SGLang & $8$ & instant & $1.0$ & -- & -- \\
\micon{openai_mark.png}gpt-oss-120b & OpenAI & SGLang & $1$ & low/medium effort & $1.0$ & $1.0$ & $0$ \\
\bottomrule
\end{tabular}
\caption{Self-hosted fixed-weight models and their serving and decoding settings. TP denotes tensor parallelism per replica, and dashes mark parameters left at their vendor defaults. Qwen instant modes use a presence penalty of $1.5$, compared with $0$ in thinking mode. For gpt-oss-120b, we vary the reasoning effort.}
\label{tab:openweight-serving}
\end{table}

\paragraph{Composite budget.}
The profiles $\{0.25,0.5,1,2\}\times$ jointly scale four hard caps relative to the $1\times$ reference in Table~\ref{tab:openweight-budget}: model calls, Lean calls, tokens, and wall-clock time. The mechanism, decoding parameters, and verification procedure remain fixed. Because all four caps change together, the study does not isolate the effect of any one resource.

\begin{table}[t]
\centering\small
\setlength{\tabcolsep}{5pt}
\begin{tabular}{@{}lrrrr@{}}
\toprule
Budget & Model calls & Lean calls & Tokens & Wall (s) \\
\midrule
$0.25\times$ & $3$ & $15$ & $100$K & $450$ \\
$0.5\times$ & $6$ & $30$ & $200$K & $900$ \\
$1\times$ & $12$ & $60$ & $400$K & $1{,}800$ \\
$2\times$ & $24$ & $120$ & $800$K & $3{,}600$ \\
\bottomrule
\end{tabular}
\caption{Per-target composite budget profiles. Each row gives four hard caps. ``Tokens'' is the combined input$+$output allowance; the per-call output cap remains $32{,}768$.}
\label{tab:openweight-budget}
\end{table}

\section{Results of Test-Time Budget Scaling}
\label{app:openweight-results}
\label{app:openweight-grid}

\paragraph{Search activity.}
Table~\ref{tab:openweight-grid} reports mean kernel-verified transitions per target for each seed in $\{19,36,65\}$. Figure~\ref{fig:openweight}a plots the target-weighted mean across the three seeds. Among recorded outcomes, seven of the eight configurations increase monotonically with budget in every seed. Low-effort gpt-oss-120b remains near zero. The number of recorded outcomes decreases for some runs in thinking mode at larger budgets. For example, Qwen3.5-397B in thinking mode has $n=1441/1428/1359/1271$ recorded target--seed outcomes at $0.25/0.5/1/2\times$. The reported values are target-weighted over recorded outcomes.

\paragraph{Parser coverage.}
The proof-state S-expression parser deliberately does not support the AST forms \texttt{:mv}, \texttt{:mvd}, \texttt{:subst}, and \texttt{:proj}. When the parser encounters one of these forms, the runner exits before writing a terminal outcome. These exits occur more often in deeper searches, so missing outcome records become more frequent as the budget grows. Means computed only from recorded outcomes may therefore be biased upward at larger budgets. We treat the curves as descriptive search activity, not as unbiased estimates for the full manifest or measures of proposal efficiency.

\begin{table}[t]
\centering\small
\setlength{\tabcolsep}{6pt}
\begin{tabular}{@{}ll cccc@{}}
\toprule
Model & Mode & $0.25\times$ & $0.5\times$ & $1\times$ & $2\times$ \\
\midrule
Qwen3.5-397B & thinking & 1.44/1.15/1.18 & 2.35/1.96/2.08 & 3.80/3.66/3.73 & 6.41/6.15/6.59 \\
Qwen3.5-397B & instant & 0.82/0.78/0.85 & 1.19/1.02/1.23 & 1.64/1.78/1.67 & 2.26/2.37/2.37 \\
Qwen3.6-35B & instant & 0.48/0.51/0.41 & 0.70/0.89/0.75 & 1.35/1.58/1.57 & 2.27/2.67/2.65 \\
Qwen3.6-35B & thinking & 0.50/0.51/0.53 & 0.81/0.82/0.75 & 1.31/1.32/1.31 & 2.21/2.19/2.17 \\
GLM-5.1 & instant & 1.05/0.95/1.06 & 1.72/1.44/1.63 & 2.57/2.63/2.78 & 4.26/4.27/4.29 \\
Kimi-K2.6 & instant & 0.97/0.89/0.98 & 1.38/1.20/1.31 & 1.95/2.12/2.17 & 3.20/2.94/3.11 \\
gpt-oss-120b & med. & 0.56/0.49/0.57 & 0.90/0.90/0.89 & 1.42/1.47/1.50 & 2.19/2.28/2.39 \\
gpt-oss-120b & low & 0.01/0.02/0.01 & 0.01/0.01/0.02 & 0.03/0.01/0.02 & 0.02/0.02/0.02 \\
\bottomrule
\end{tabular}
\caption{Mean kernel-verified transitions per target, with entries ordered as \textbf{seed 19 / seed 36 / seed 65}. Figure~\ref{fig:openweight}a plots the target-weighted mean across the three seeds. These counts measure cumulative verified search activity, not unique proofs or solve rates.}
\label{tab:openweight-grid}
\end{table}

\paragraph{Solves.}
Table~\ref{tab:openweight-solves} reports both solve events and distinct target coverage. Between $0.25\times$ and $2\times$, the number of solve events rises from $28$ to $44$, while the union of solved targets rises from $2$ to $10$. Kimi-K2.6 and the two Qwen3.5 modes account for most of the endpoint increase. Four targets are solved at $2\times$ but not at a lower budget: \texttt{brualdi\_ch10\_31}, \texttt{brualdi\_ch1\_10}, \texttt{hackmath\_4}, and \texttt{putnam\_1977\_a5}. The budget arms are separate stochastic runs, and the number of recorded outcomes varies with budget. These counts describe observed coverage; they do not establish monotone per-model scaling or isolate the effect of any one resource.

\begin{table}[t]
\centering\small
\setlength{\tabcolsep}{4.5pt}
\begin{tabular}{@{}llrrrr@{}}
\toprule
Model & Mode & $0.25\times$ & $0.5\times$ & $1\times$ & $2\times$ \\
\midrule
Qwen3.5-397B & thinking & $3$ & $4$ & $5$ & $9$ \\
Qwen3.5-397B & instant & $3$ & $3$ & $3$ & $6$ \\
Qwen3.6-35B  & instant & $3$ & $3$ & $4$ & $3$ \\
Qwen3.6-35B  & thinking & $4$ & $3$ & $3$ & $4$ \\
GLM-5.1      & instant & $3$ & $4$ & $5$ & $3$ \\
Kimi-K2.6    & instant & $3$ & $7$ & $9$ & $11$ \\
gpt-oss-120b & medium  & $6$ & $5$ & $6$ & $6$ \\
gpt-oss-120b & low     & $3$ & $3$ & $0$ & $2$ \\
\midrule
\multicolumn{2}{@{}l}{All eight configurations: solve events} & $28$ & $32$ & $35$ & $44$ \\
\multicolumn{2}{@{}l}{Union of solved targets} & $2$ & $6$ & $6$ & $10$ \\
\bottomrule
\end{tabular}
\caption{Final kernel-verified outcomes across seeds $\{19,36,65\}$. A solve event is one run for a particular model, mode, seed, and target that returns a valid proof. Repeated solutions of the same target count separately. The last row counts target identifiers solved at least once. Budget columns correspond to separate runs rather than cumulative prefixes.}
\label{tab:openweight-solves}
\end{table}

\paragraph{Solved-target inventory.}
Across all budgets and seeds, the open-weight runs solve $11$ distinct targets: $4$ from PutnamBench and $7$ from CombiBench. None is from IMO-LeanProofBench. Table~\ref{tab:openweight-inventory} gives the model, budget, and search statistics for one representative run per target. All five models solve \texttt{brualdi\_ch14\_33}. The listed solves for \texttt{putnam\_2012\_a2} and \texttt{putnam\_1977\_a5} occur only at the larger budgets.

\begin{table}[t]
\centering\footnotesize
\setlength{\tabcolsep}{5pt}
\begin{tabular}{@{}lll rrr p{3.6cm}@{}}
\toprule
Target & Bench & Model (mode, budget) & Trans. & Calls & Lean & Representative mechanism \\
\midrule
\texttt{putnam\_2012\_a2} & Putnam & Qwen3.5-397B (thinking, $2\times$) & $7$ & $17$ & $40$ & decomposition and derived identity lemma \\
\texttt{putnam\_1977\_a5} & Putnam & Qwen3.5-397B (thinking, $2\times$) & $7$ & $23$ & $59$ & ProofLib schema reuse \\
\texttt{putnam\_2001\_a1} & Putnam & Qwen3.5-397B (thinking, $2\times$) & $3$ & $5$  & $21$ & hypothesis instantiation and rewrite \\
\texttt{putnam\_1988\_b1} & Putnam & Qwen3.5-397B (thinking, $2\times$) & $2$ & $5$  & $10$ & explicit construction and \texttt{ring} \\
\texttt{hackmath\_4}      & Combi  & Kimi-K2.6 (instant, $2\times$) & $1$ & $3$  & $8$  & \texttt{IsLeast} decomposition and pigeonhole \\
\texttt{brualdi\_ch1\_10} & Combi  & Kimi-K2.6 (instant, $2\times$) & $4$ & $14$ & $33$ & order-$2$ impossibility via \texttt{omega} \\
\texttt{brualdi\_ch7\_7}  & Combi  & Kimi-K2.6 (instant, $2\times$) & $3$ & $3$  & $28$ & \texttt{Int.gcd\_fib} rewrite \\
\texttt{brualdi\_ch8\_6}  & Combi  & Kimi-K2.6 (instant, $2\times$) & $1$ & $3$  & $11$ & induction on the summation \\
\texttt{brualdi\_ch14\_33}& Combi  & Kimi-K2.6 (instant, $2\times$) & $1$ & $1$  & $10$ & single-lemma rewrite (\texttt{cycleType\_inv}) \\
\texttt{brualdi\_ch2\_11} & Combi  & GLM-5.1 (instant, $1\times$)   & $1$ & $2$  & $7$  & kernel \texttt{decide} \\
\texttt{brualdi\_ch10\_31}& Combi  & Qwen3.6-35B (thinking, $2\times$) & $2$ & $22$ & $30$ & witness and kernel \texttt{decide} \\
\bottomrule
\end{tabular}
\caption{The $11$ distinct targets solved in the open-weight runs. For one representative run per target, ``Trans.''\ gives kernel-verified transitions, ``Calls''\ gives model calls, and ``Lean''\ gives kernel calls. Every listed run reaches $\rho=1$. Medium-effort gpt-oss-120b also solves \texttt{brualdi\_ch14\_33} and \texttt{brualdi\_ch7\_7}. Appendix~\ref{app:openweight-cases} gives the corresponding proofs.}
\label{tab:openweight-inventory}
\end{table}

\paragraph{Re-verification.}
We re-verify every solve in the frozen Lean~4.29.1 environment with Mathlib commit \texttt{5e932f97}. Each returned proof term is elaborated from scratch, and \texttt{\#print axioms} enumerates its axiom dependencies. All $139$ kernel-verified solve events have a stored closing proof, with no reported $\rho=1$ lacking one. Every proof depends only on $\{\texttt{propext},\ \texttt{Classical.choice},\ \texttt{Quot.sound}\}$, and none uses \texttt{native\_decide}. Re-verification found $0$ false positives.

\paragraph{Verified substrate.}
In addition to Mathlib, the ProofEvolve environment $\Env$ contains the library of $158$ kernel-verified lemmas. The library has no occurrences of \texttt{sorry} and introduces no axioms beyond the standard three. A proof that uses one of these lemmas therefore has the same axiom footprint as a proof built directly on Mathlib. During re-verification, we inline the library so that \texttt{\#print axioms} checks its dependencies transitively. The proof of \texttt{putnam\_1977\_a5} in Appendix~\ref{app:openweight-cases} uses one lemma from this library.

\paragraph{Comparison with agentic baselines.}
We also evaluate pass@$16$, escalating whole-proof search, ReAct, LEAP, Hilbert, and ProofEvolve on a fixed random subset shared across models. All methods solve few targets with these open-weight models, so we draw no quantitative cross-method conclusion from this subset. Verified closure and kernel-verified transition counts are specific to ProofEvolve's DAG and do not support a comparison with the other methods. Table~\ref{tab:main}, which uses Claude Opus~4.8 on the full benchmarks, provides the matched solve-rate comparison.

\section{Kernel-Certified Open-Weight Proofs on PutnamBench and CombiBench}
\label{app:openweight-cases}
This appendix gives the kernel-certified proofs for all $11$ distinct open-weight solves in Table~\ref{tab:openweight-inventory}. We verify them as described in Appendix~\ref{app:openweight-results}. Each theorem name is its benchmark target identifier. The caption records the solving model and the kernel-verified transitions and model calls for one solving run.

\subsection{PutnamBench}

\begin{LeanBox}[label={lst:putnam2012a2}]{putnam\_2012\_a2: Qwen3.5-397B (thinking, $2\times$), $7$ transitions / $17$ calls}
import Mathlib
import ProofLib

open Matrix

theorem putnam_2012_a2
(S : Type*) [CommSemigroup S]
(a b c : S)
(hS : ∀ x y : S, ∃ z : S, x * z = y)
(habc : a * c = b * c)
: a = b := by
  obtain ⟨z, hz⟩ := hS c a
  obtain ⟨e, he⟩ := hS c c
  have h_id : ∀ x : S, x * e = x := by
    intro x
    obtain ⟨w, hw⟩ := hS c x
    calc
      x * e = (c * w) * e := by rw [hw]
      _ = c * (w * e) := by rw [mul_assoc]
      _ = c * (e * w) := by rw [mul_comm w e]
      _ = (c * e) * w := by rw [mul_assoc]
      _ = c * w := by rw [he]
      _ = x := by rw [hw]
  obtain ⟨d, hd⟩ := hS c e
  exact calc
    a = a * e := by rw [h_id a]
    _ = a * (c * d) := by rw [hd]
    _ = (a * c) * d := by rw [mul_assoc]
    _ = (b * c) * d := by rw [habc]
    _ = b * (c * d) := by rw [mul_assoc]
    _ = b * e := by rw [hd]
    _ = b := by rw [h_id b]
\end{LeanBox}
right cancellation in a commutative semigroup satisfying the stated divisibility condition. The proof first establishes $\forall x,\ x*e=x$ and then uses this identity in the cancellation argument.

\begin{LeanBox}[label={lst:putnam1977a5}]{putnam\_1977\_a5: Qwen3.5-397B (thinking, $2\times$), $7$ transitions / $23$ calls}
import Mathlib
import ProofLib

open RingHom Set Nat

theorem putnam_1977_a5
(p m n : ℕ)
(hp : Nat.Prime p)
(hmgen : m ≥ n)
: (choose (p * m) (p * n) ≡ choose m n [MOD p]) := by
  rw [Nat.ModEq]
  haveI : Fact (Nat.Prime p) := ⟨hp⟩
  rw [Choose.choose_modEq_choose_mod_mul_choose_div_nat]
  simp [Nat.mul_mod, Nat.mul_div_cancel_left, Nat.choose_zero_right, Nat.mod_eq_of_lt]
  simp [Nat.mul_div_cancel_left, Nat.Prime.pos hp]
\end{LeanBox}
the Lucas-type congruence $\binom{pm}{pn}\equiv\binom{m}{n}\pmod p$. The proof applies a $p$-adic binomial congruence lemma from ProofLib and finishes with \texttt{simp}.

\begin{LeanBox}[label={lst:putnam1988b1}]{putnam\_1988\_b1: Qwen3.5-397B (thinking, $2\times$), $2$ transitions / $5$ calls}
import Mathlib
import ProofLib

open Set Filter Topology

theorem putnam_1988_b1
: ∀ a ≥ 2, ∀ b ≥ 2, ∃ x y z : ℤ, x > 0 ∧ y > 0 ∧ z > 0 ∧ a * b = x * y + x * z + y * z + 1 := by
  intro a ha b hb; use 1, a - 1, b - 1; constructor; norm_num; constructor; linarith; constructor; linarith; ring
\end{LeanBox}
every product $ab$ with $a,b\ge 2$ equals $xy+xz+yz+1$ for positive integers $x,y,z$. The proof sets $x=1$, $y=a-1$, and $z=b-1$, then finishes with \texttt{linarith} and \texttt{ring}.

\begin{LeanBox}[label={lst:putnam2001a1}]{putnam\_2001\_a1: Qwen3.5-397B (thinking, $2\times$), $3$ transitions / $5$ calls}
import Mathlib
import ProofLib

open Topology Filter

theorem putnam_2001_a1
(S : Type*)
[Mul S]
(hS : ∀ a b : S, (a * b) * a = b)
: ∀ a b : S, a * (b * a) = b := by
  intro a b
  have h := hS b a; have h2 := hS (b * a) b; rw [h] at h2; exact h2
\end{LeanBox}
if a magma satisfies $(a*b)*a=b$ for all $a,b$, then $a*(b*a)=b$. The proof instantiates the hypothesis twice and rewrites once.

\subsection{CombiBench}

\begin{LeanBox}[label={lst:hackmath4}]{hackmath\_4: Kimi-K2.6 (instant, $2\times$), $1$ transitions / $3$ calls}
import Mathlib

theorem hackmath_4 : IsLeast {n | ∀ f : Fin n → Fin 12, ∃ a b, a ≠ b ∧ f a = f b} ((13) : ℕ ) := by
  constructor
  · -- Show 13 is in the set
    intro f
    have h1 : Fintype.card (Fin 13) = 13 := by simp
    have h2 : Fintype.card (Fin 12) = 12 := by simp
    have h3 : Fintype.card (Fin 13) > Fintype.card (Fin 12) := by rw [h1, h2]; norm_num
    obtain ⟨a, b, hab, heq⟩ := Fintype.exists_ne_map_eq_of_card_lt f h3
    exact ⟨a, b, hab, heq⟩
  · -- Show 13 is the least element in the set
    intro n hn
    by_contra h
    push_neg at h
    have h4 : n ≤ 12 := by omega
    have h5 : ∃ f : Fin n ↪ Fin 12, True := by
      have h6 : Fintype.card (Fin n) ≤ Fintype.card (Fin 12) := by
        simp
        omega
      have h7 : Nonempty (Fin n ↪ Fin 12) := Function.Embedding.nonempty_of_card_le h6
      obtain ⟨f⟩ := h7
      exact ⟨f, trivial⟩
    obtain ⟨f, _⟩ := h5
    have h7 : ∀ a b, a ≠ b → f a ≠ f b := by
      intro a b hab
      exact f.injective.ne hab
    have h8 := hn f
    obtain ⟨a, b, hab, heq⟩ := h8
    have h9 : f a ≠ f b := h7 a b hab
    contradiction
\end{LeanBox}
$13$ is the least group size that forces two people into the same month. The proof establishes membership with \texttt{Fintype.exists\_ne\_map\_eq\_of\_card\_lt} and minimality by constructing an embedding for $n\le 12$.

\begin{LeanBox}[label={lst:brualdich110}]{brualdi\_ch1\_10: Kimi-K2.6 (instant, $2\times$), $4$ transitions / $14$ calls}
import Mathlib

structure IsMagicSquare {n : ℕ} (M : Matrix (Fin n) (Fin n) ℕ) : Prop where
  mem : ∀ i j, M i j ∈ Finset.Icc 1 (n * n)
  pairwise : ∀ i j i' j', i ≠ i' ∨ j ≠ j' → M i j ≠ M i' j'
  same_sum : ∃ s, (∀ i, ∑ j, M i j = s) ∧ (∀ j, ∑ i, M i j = s) ∧ (∑ i, M i i.rev = s) ∧ ∑ i, M i i = s

theorem brualdi_ch1_10 : ¬∃ (M : Matrix (Fin 2) (Fin 2) ℕ), IsMagicSquare M := by
  intro h
  rcases h with ⟨M, hM⟩
  rcases hM with ⟨h_rows, h_cols, h_diag, h_anti, h_pairwise⟩
  have h1 := h_rows 0 0
  have h2 := h_rows 0 1
  have h3 := h_rows 1 0
  have h4 := h_rows 1 1
  have h5 := h_cols 0 0 0 1 (Or.inr (by decide))
  have h6 := h_cols 0 0 1 0 (Or.inl (by decide))
  have h7 := h_cols 0 0 1 1 (Or.inl (by decide))
  have h8 := h_cols 0 1 1 0 (Or.inl (by decide))
  have h9 := h_cols 0 1 1 1 (Or.inl (by decide))
  have h10 := h_cols 1 0 1 1 (Or.inr (by decide))
  simp [Finset.mem_Icc] at h1 h2 h3 h4
  have h_row0 := h_anti 0
  have h_row1 := h_anti 1
  have h_col0 := h_pairwise.1 0
  have h_col1 := h_pairwise.1 1
  have h_diag1 := h_pairwise.2.2
  have h_anti_diag := h_pairwise.2.1
  simp [Fin.sum_univ_two, Fin.rev] at h_row0 h_row1 h_col0 h_col1 h_diag1 h_anti_diag
  omega
\end{LeanBox}
no order-$2$ magic square exists. The proof states the range, distinctness, and row, column, and diagonal sum constraints, then solves the resulting integer system with \texttt{omega}.

\begin{LeanBox}[label={lst:brualdich77}]{brualdi\_ch7\_7: Kimi-K2.6 (instant, $2\times$), $3$ transitions / $3$ calls}
import Mathlib

theorem brualdi_ch7_7 (m n d : ℕ+) (hmd : d = Nat.gcd m n) :
    Nat.gcd (Nat.fib m) (Nat.fib n) = Nat.fib d := by
  have h1 : (Nat.fib ↑m).gcd (Nat.fib ↑n) = Nat.fib (Int.gcd (↑m : ℤ) (↑n : ℤ)) := by rw [← Int.gcd_fib (↑m : ℤ) (↑n : ℤ)]; simp
  rw [hmd]
  simp [hmd] at h1 ⊢; exact h1
\end{LeanBox}
$\gcd(F_m,F_n)=F_{\gcd(m,n)}$ for Fibonacci numbers. The proof converts between the $\mathbb{N}$ and $\mathbb{Z}$ formulations and applies the library identity.

\begin{LeanBox}[label={lst:brualdich86}]{brualdi\_ch8\_6: Kimi-K2.6 (instant, $2\times$), $1$ transitions / $3$ calls}
import Mathlib

theorem brualdi_ch8_6 (n : ℕ) (h : ℕ → ℝ) (h' : ∀ i, h i = 2 * i ^ 2 - i + 3) :
    ∑ i ∈ Finset.range (n + 1), h i = ((fun n => ((n + 1) * (4 * n ^ 2 - n + 18) / 6)) : ℕ → ℝ ) n := by
  induction n with
  | zero =>
    simp [Finset.sum_range_succ, h']
    all_goals norm_num
  | succ n ih =>
    rw [Finset.sum_range_succ, ih]
    simp [h']
    ring_nf
    <;> field_simp
    <;> ring_nf
    <;> norm_num
    <;> ring
\end{LeanBox}
the closed form $\sum_{k=0}^{n}(2k^2-k+3)=\tfrac{(n+1)(4n^2-n+18)}{6}$, proved by induction.

\begin{LeanBox}[label={lst:brualdich1433}]{brualdi\_ch14\_33: Kimi-K2.6 (instant, $2\times$), $1$ transitions / $1$ calls}
import Mathlib

theorem brualdi_ch14_33 {α : Type*} [Fintype α] [DecidableEq α] (σ : Equiv.Perm α) :
    σ.cycleType = σ⁻¹.cycleType := by
  rw [Equiv.Perm.cycleType_inv]
\end{LeanBox}
a permutation and its inverse have the same cycle type. The proof applies the library identity \texttt{Equiv.Perm.\allowbreak cycleType\_inv}; all five models solve this target.

\begin{LeanBox}[label={lst:brualdich211}]{brualdi\_ch2\_11: GLM-5.1 (instant, $1\times$), $1$ transitions / $2$ calls}
import Mathlib

open Finset

theorem brualdi_ch2_11 :
    ((Icc (1 : ℕ) 20).powersetCard 3 |>.filter (fun S => ∀ a ∈ S, a - 1 ∉ S ∧ a + 1 ∉ S)).card =
    ((816) : ℕ ) := by
  set_option maxRecDepth 1000000 in decide
\end{LeanBox}
there are $816$ size-$3$ subsets of $\{1,\dots,20\}$ with no two consecutive elements. The kernel tactic \texttt{decide} verifies the count without \texttt{native\_decide}.

\begin{LeanBox}[label={lst:brualdich1031}]{brualdi\_ch10\_31: Qwen3.6-35B (thinking, $2\times$), $2$ transitions / $22$ calls}
import Mathlib

def isDifferenceSet (n : ℕ) (B : Finset (ZMod n)) : Prop :=
  ∃ k, ∀ x : (ZMod n),  x ≠ 0 → ∑ i ∈ B, ∑ j ∈ B \ {i}, List.count x [i - j] = k

theorem brualdi_ch10_31 : isDifferenceSet 21 {0, 3, 4, 9, 11} := by
  dsimp only [isDifferenceSet]
  use 1; decide
\end{LeanBox}
$\{0,3,4,9,11\}$ is a difference set in $\mathbb{Z}_{21}$. The proof supplies the multiplicity witness and uses the kernel tactic \texttt{decide}, not \texttt{native\_decide}.

\clearpage
\section{Detailed Setup for the Lean Workbook Study}
\label{app:heldout}

This appendix records the corpus, library construction, screening, retrieval and
evaluation used in Section~\ref{sec:transfer}.

\subsection{Corpus and verification}

We use theorem statements from Lean~Workbook \citep{ying2024leanworkbook} and the
machine-generated proofs released by Goedel-Prover \citep{lin2025goedelproverv1}. We
re-elaborate every candidate in a fixed Lean~4 environment \citep{demoura2021lean4}
with a matched Mathlib commit \citep{mathlib2020}, under the same acceptance criteria
used throughout the paper: an accepted proof has no unresolved metavariables or
placeholders, passes the Lean kernel at the stated type, and a restricted
\texttt{\#print axioms} check must show dependencies only on \texttt{propext},
\texttt{Classical.choice} and \texttt{Quot.sound}. We reject proofs that use
\texttt{native\_decide} because its code generation path introduces an axiom outside
the standard kernel. Each candidate runs in a fresh process with a $300$-second
timeout. This verification retains $20{,}554$ statement and proof pairs.

\subsection{Deduplication and the evaluation split}

Lean~Workbook contains many restatements of the same theorem, so we deduplicate
before splitting. We generate candidate pairs with MinHash LSH over statement
shingles using $128$ permutations, $32$ bands of $4$ rows, seed $20260806$ and an
approximate threshold of $0.42$. This produces $2{,}362{,}946$ candidate pairs.
We confirm each pair with direct similarity tests or a guarded signature match.
The direct tests use alpha equivalence of the elaborated statement, statement
$n$-gram Jaccard above $0.8$ and docstring $n$-gram Jaccard above $0.8$. The
guarded match requires signature Jaccard above $0.95$ over at least five symbols and
statement Jaccard above $0.6$. We disable the dense embedding channel for
deduplication because transitive embedding matches can join distinct theorems into
one cluster. The exact tests retain $26{,}330$ edges and $5{,}216$ multi-theorem
clusters. The largest cluster contains $106$ theorems. Removing $6{,}935$
alpha-exact duplicates and $2{,}651$ near duplicates leaves $10{,}968$
representatives, a $46.6\%$ reduction. A fixed seed assigns $1{,}000$ theorems to
evaluation and $9{,}968$ to the source stream. A post-split audit checks the boundary
again for residual near duplicates.

\subsection{Leakage screening}

The exact tests can still miss the same theorem written in a different form.
Five independent language-model judges therefore screen each evaluation theorem
against its $64$ nearest library neighbors. A judge flags a residual near duplicate
when a neighbor states the same theorem up to renaming or a change of constants.
We record every verdict and its reason, then exclude a theorem when at least two
judges flag it. The screen removes $256$ theorems, or $25.6\%$, and leaves the
$744$ used throughout the study. Screening occurs before condition outcomes are
compared, and every condition uses the same evaluation set.

The panel is five Claude models spanning four releases
(\texttt{claude-opus-5}, \texttt{claude-opus-4.8}, \texttt{claude-opus-4.7},
\texttt{claude-opus-4.6} and \texttt{claude-sonnet-5}), each queried once per
evaluation theorem at medium reasoning effort. Verdicts are independent across models
but not across model families. Flag counts over the $1{,}000$ candidates are $624$,
$120$, $57$, $37$, $56$ and $106$ for zero through five flags. Agreement is therefore
bimodal: $730$ theorems, or $73.0\%$, receive a unanimous verdict, and only $94$ fall
in the two-to-three-flag band where the threshold is decisive. Raising the threshold
to three flags would retain $801$ theorems instead of $744$. Prompt~\ref{prompt:leak}
gives the template; each judge sees only the evaluation theorem and its retrieved
neighbors.

\begin{PromptBox}[label={prompt:leak}]{Leakage judge, one call per evaluation theorem}
You are a LEAKAGE auditor for a held-out Lean theorem-proving benchmark. Read
ONLY the record for this evaluation theorem (a single JSON object with
holdout_statement, holdout_docstring, and a "retrieved" array -- each item has
statement + proof_body + docstring). Do NOT read any other file.

Rule LEAK if ANY retrieved item is a NEAR-DUPLICATE of the holdout (the same
problem up to variable renaming / constant changes / trivial rewrite) OR
directly contains the holdout's answer or a proof that would transfer by
copying (memorization, not reasoning).

Rule CLEAN if the retrieved items are only related-but-distinct (same technique
/ topic, genuinely different problem).

Give the verdict + a one-sentence reason.
\end{PromptBox}

\subsection{Verified theorem schema library}

The base model attempts every buildable theorem in the $9{,}968$-theorem source
stream once under the same one-shot budget used for evaluation. Target files
cannot be built for $384$ theorems. Of the remaining $9{,}584$, the kernel and
axiom audit accept $5{,}546$ proofs. The self-solve rate is $57.9\%$ over scoreable
theorems and $55.6\%$ over the full source stream. These $5{,}546$ theorem and
proof pairs form the persistent verified theorem schema library. Every entry
contains a proof written by the base model and checked by Lean. No human proof or
released Goedel-Prover proof enters the library. Unlike ProofLib, whose lemmas are
installed in $\Env$ and therefore enter the axiom footprint of any proof that uses
them (Appendix~\ref{app:openweight-results}), this library is a retrieval corpus
only: its entries never enter the Lean environment and never appear in any proof's
axiom footprint. The library-size conditions use nested subsets drawn once with a
fixed seed, so the $1{,}000$-proof library is contained in the $2{,}000$-proof
library, which is contained in the $4{,}000$-proof library.

\subsection{Semantic retrieval}

The semantic retriever embeds theorem statements with
\texttt{bge-large-en-v1.5}. Raw cosine similarities occupy a narrow band because
the corpus contains many competition-style algebra and inequality theorems. We
subtract the library mean from each vector before computing cosine similarity.
The encoder uses asymmetric instructions. Evaluation queries carry the prefix
\texttt{Represent this sentence for searching relevant passages:}, while library
entries are embedded without a prefix. Cosine scores are rounded to two decimal
places, with lexical statement Jaccard breaking ties. The rank excludes the symbol
and type signature channel because it saturates at $1.0$ on this corpus.
Retrieval is deterministic, so each run receives the same $K$ library entries for
a given theorem.

\subsection{Base model, serving and decoding}

The base model is Qwen3.5-397B-A17B in FP8 \citep{qwen2026qwen35fp8card}, the same
open-weight model used in the test-time scaling study of
Appendix~\ref{app:openweight}, but served differently here. We serve it with SGLang
rather than vLLM, at tensor-parallel size $8$ rather than $4$. Each replica uses
eight B200 GPUs, a context window of $65{,}536$ tokens and GPU memory utilization of
$0.92$. A client distributes requests across $16$ replicas and retries transport
errors on the next live replica. The model runs in its default thinking mode without
a chat-template argument. Library construction and all evaluation conditions use one
decoding profile pinned by hash: temperature $0.6$, top-$p$ $0.95$, top-$k$ $20$,
min-$p$ $0.0$, presence penalty $0.0$, repetition penalty $1.0$, and a client-side
limit of $32{,}768$ output tokens. We do not pin a decode seed. The three runs are
independent samples from the decoder, and the reported standard deviation measures
variation across these runs.

\subsection{Prompt}

All three conditions use one template containing a system prefix, the target, base
premises, worked examples and previous attempts. Only the worked-examples section
changes. Zero-shot renders this section as \texttt{None available.} rather than
omitting it. Base premises and previous attempts are empty in every condition, so
the schema library is the only source of retrieved proof knowledge. Each schema is
rendered as a compilable Lean \texttt{example} with its theorem statement and the
proof body accepted by Lean.

\begin{PromptBox}{ProofEvolve prompt with relevant schemas}
[system]
You are proving a theorem in Lean 4 with Mathlib.

Reply with a single fenced ```lean block containing ONLY the tactic block
that completes the given theorem. Do not restate the theorem, do not add
imports, and do not declare anything at the top level: the surrounding file
is fixed and your reply is inserted after `:= by`.

If you cannot close the goal, still reply with your best tactic block.
Never emit `sorry`.

[user]
## Target
The file below is frozen. Your tactic block is inserted where marked.
```lean
import Mathlib

theorem <name> <statement> := by
  <your tactic block here>
```
## Proof state
One open obligation: the theorem statement above.
## Base premises
None selected for this request.
## Worked examples
Solved problems from your own library, for reference. They are different
problems; adapt the techniques, do not copy.
### Example 1
```lean
example <library statement 1> := by
<verified proof body 1>
```
... (K examples)
## Previous attempts
None.
\end{PromptBox}

\subsection{Conditions, budget and compute}

Every condition gives one attempt per theorem with no self-repair or second sample.
The Lean environment imports Mathlib. Relevant retrieval uses
$K\in\{8,16,32,64\}$ over the full library. We also evaluate library sizes
$\{1{,}000,\,2{,}000,\,4{,}000,\,\text{full}\}$ at $K{=}8$. Random retrieval uses
$K{=}8$ with a fixed draw seed for each run. We run all nine configurations three
times over the same $1{,}000$ evaluation candidates, then score the common
$744$-theorem screened subset in every condition. Every run produces one record
per candidate. Condition contrasts are computed as paired differences across matched
runs.

For the copy rate in Section~\ref{sec:transfer}, we remove comments and normalize
whitespace in every accepted proof body. We then compare it with the eight proof
bodies shown for that theorem. For each run index, we count cases where relevant
retrieval closes a theorem and the corresponding zero-shot run leaves it open.

The fleet contains $16$ FP8 replicas of $8$ B200 GPUs each, or $128$ B200 GPUs in
total. Building the library over $9{,}968$ source theorems takes about five hours
of fleet time. A three-run sweep of one configuration takes about seven hours. The
complete study uses a few thousand B200-GPU-hours across library construction,
the depth and size sweeps and leakage screening.

\begin{table}[t]
\centering
\setlength{\tabcolsep}{5pt}
\begin{tabular}{@{}lcc@{}}
\toprule
Condition & Solve rate (\%) & Lift (pp) \\
\midrule
Zero-shot & $49.5 \pm 0.7$ & $0.0$ \\
Random retrieval ($K{=}8$) & $49.6 \pm 1.0$ & $+0.1$ \\
\midrule
\multicolumn{3}{c}{\itshape Relevant retrieval, depth (full library)} \\
$K{=}8$  & $53.4 \pm 0.7$ & $+3.9$ \\
$K{=}16$ & $54.1 \pm 0.9$ & $+4.6$ \\
$K{=}32$ & $52.9 \pm 0.9$ & $+3.4$ \\
$K{=}64$ & $54.8 \pm 0.5$ & $+5.3$ \\
\midrule
\multicolumn{3}{c}{\itshape Relevant retrieval, library size ($K{=}8$)} \\
$1{,}000$ proofs & $52.8 \pm 0.3$ & $+3.3$ \\
$2{,}000$ proofs & $52.2 \pm 0.8$ & $+2.6$ \\
$4{,}000$ proofs & $53.7 \pm 0.3$ & $+4.2$ \\
Full library ($5{,}546$) & $53.4 \pm 0.7$ & $+3.9$ \\
\bottomrule
\end{tabular}
\caption{Solve rate on the $744$ screened evaluation theorems. We report the mean
over three independent runs, and the $\pm$ figure is the run-to-run standard
deviation \emph{of the solve rate}. Lift is the paired difference from zero-shot,
computed per run before rounding, so it need not equal the difference of the two
rounded means.}
\label{tab:heldout}
\end{table}

\section{Retrieval-to-Proof Traces on Lean Workbook}
\label{app:examples}

Section~\ref{sec:transfer} reports that $322$ of the $351$ theorems closed by
relevant retrieval and left open by zero-shot, or $91.7\%$, do not reproduce any
shown proof verbatim. That is an aggregate. This appendix gives the individual form
of it: two complete traces from the evaluation theorem, through what the retriever
actually returned, to the proof Lean accepted. Each trace runs through five stages.
\textbf{(1)}~the evaluation theorem and why it is not immediate;
\textbf{(2)}~two of the eight schemas the retriever placed in the proposal context;
\textbf{(3)}~the proof the model produced and the kernel accepted;
\textbf{(4)}~a comparison with the reference proof released with Lean~Workbook, which
shows the two take different routes; and
\textbf{(5)}~the outcome of all three conditions over the three runs.
The selected Lean statements preserve the informal problem and avoid truncated
natural-number arithmetic, inconsistent assumptions and trivial goals.

\begin{ExampleBox}{A bound on $[0,1]$}
\textbf{(1)~Target.} For $0 \le x \le 1$, show that
$|x(x-1)(x^6+2x^4+3x^2+4)| < 5/2$. The two natural factor bounds do not prove the
strict inequality: multiplying $x(1-x) \le 1/4$ and
$x^6+2x^4+3x^2+4 \le 10$ gives only $\le 5/2$.

\smallskip
\begin{lstlisting}[style=lean]
(x : ℝ) (hx : 0 ≤ x ∧ x ≤ 1) :
  |x * (x - 1) * (x^6 + 2 * x^4 + 3 * x^2 + 4)| < 5 / 2
\end{lstlisting}

\textbf{(2)~Retrieved.} Two of the eight schemas placed in the proposal context.
Both are bounds on the same interval, and neither states the target.
\begin{lstlisting}[style=lean]
example (x : ℝ) (hx : 0 ≤ x ∧ x ≤ 1) :
    0 ≤ x - x^2 ∧ x - x^2 ≤ 1/4 := by
  cases hx
  constructor <;> nlinarith [sq_nonneg (x - 1/2)]

example (x : ℝ) (hx : 0 ≤ x ∧ x ≤ 1) :
    x * (x ^ 3 - 6 * x + 9) ≤ 4 := by
  nlinarith [sq_nonneg (x - 1), sq_nonneg (x^2 - 1), sq_nonneg (x^2 - x),
    mul_nonneg hx.1 (sq_nonneg (x - 1)),
    mul_nonneg (sub_nonneg.mpr hx.2) (sq_nonneg (x - 1))]
\end{lstlisting}

\textbf{(3)~Generated.} The kernel-accepted proof. It determines the sign of the
product first, which neither retrieved schema does, and then reuses the
$\mathrm{sq\_nonneg}(x-1/2)$ hint that both of them turn on.
\begin{lstlisting}[style=lean]
cases hx with
| intro hx1 hx2 =>
  have h1 : x * (x - 1) ≤ 0 := by nlinarith
  have h2 : x^6 + 2 * x^4 + 3 * x^2 + 4 ≥ 4 := by
    nlinarith [pow_nonneg hx1 2, pow_nonneg hx1 4, pow_nonneg hx1 6]
  have h3 : x * (x - 1) * (x^6 + 2*x^4 + 3*x^2 + 4) ≤ 0 := by
    nlinarith [pow_nonneg hx1 2, pow_nonneg hx1 4, pow_nonneg hx1 6]
  rw [abs_of_nonpos h3]
  nlinarith [sq_nonneg (x - 1/2), pow_nonneg hx1 2, pow_nonneg hx1 3,
    pow_nonneg hx1 4, pow_nonneg hx1 5, pow_nonneg hx1 6,
    mul_nonneg hx1 (sq_nonneg (x - 1/2)),
    mul_nonneg (sub_nonneg.mpr hx2) (sq_nonneg (x - 1/2))]
\end{lstlisting}

\textbf{(4)~Compared.} The reference proof released with Lean~Workbook splits the
absolute value into two inequalities without determining the sign of the product.
None of its auxiliary inequalities appears in the accepted proof.
\begin{lstlisting}[style=lean]
rw [abs_lt]
constructor <;>
  nlinarith [pow_nonneg (sub_nonneg.mpr hx.1) 0, ...,
             pow_nonneg (sub_nonneg.mpr hx.1) 9]
\end{lstlisting}

\textbf{(5)~Outcome over three runs.} Zero-shot closes zero runs, relevant retrieval
closes two and random retrieval closes one.
\end{ExampleBox}

\enlargethispage{\baselineskip}
\ExampleNeedspace
\begin{ExampleBox}{A three-part symmetric conclusion}
\textbf{(1)~Target.} If $x,y,z>0$ and $x^4+y^4+z^4=3$, then $x+y+z \le 3$,
$xy+yz+zx \le 3$ and $xyz \le 1$. The conclusion is a conjunction, so a single
inequality tactic has to discharge three different bounds at once.

\smallskip
\begin{lstlisting}[style=lean]
(x y z : ℝ) (hx : 0 < x) (hy : 0 < y) (hz : 0 < z)
    (h : x^4 + y^4 + z^4 = 3) :
  x + y + z ≤ 3 ∧ x*y + y*z + z*x ≤ 3 ∧ x*y*z ≤ 1
\end{lstlisting}

\textbf{(2)~Retrieved.} Two of the eight schemas. Both are symmetric
three-variable inequalities under a different constraint.
\begin{lstlisting}[style=lean]
example (x y z : ℝ) (hx : 0 ≤ x) (hy : 0 ≤ y) (hz : 0 ≤ z)
    (h : x^2 + y^2 + z^2 + x*y*z = 4) :
    3*(x^2*y + y^2*z + z^2*x) ≤ (x + y + z)*(x^2 + y^2 + z^2) := by
  nlinarith [sq_nonneg (x - y), sq_nonneg (y - z), sq_nonneg (z - x),
    mul_nonneg hx hy, mul_nonneg hy hz, mul_nonneg hz hx,
    mul_nonneg (sq_nonneg (x - y)) hz, ...]

example (x y z : ℝ) (hx : 0 ≤ x ∧ x ≤ 1) (hy : 0 ≤ y ∧ y ≤ 1)
    (hz : 0 ≤ z ∧ z ≤ 1) :
    2 * (x*y + y*z + z*x) ≤ 3*x*y*z + x + y + z := by
  have h1 : 0 ≤ (1 - x) * (1 - y) * z := by
    apply mul_nonneg
    · apply mul_nonneg <;> linarith
    · linarith
  ...
\end{lstlisting}

\textbf{(3)~Generated.} The kernel-accepted proof builds eleven named intermediate
results and assembles the conjunction at the end, rather than attacking it in one
step. Supporting expressions are elided at \texttt{...} for space.
\begin{lstlisting}[style=lean]
have h1 : x^2 + y^2 + z^2 ≤ 3 := by
  nlinarith [sq_nonneg (x^2 - 1), ..., h]
have h2 : x*y + y*z + z*x ≤ x^2 + y^2 + z^2 := by
  nlinarith [sq_nonneg (x - y), sq_nonneg (y - z), sq_nonneg (z - x)]
have h3 : x*y + y*z + z*x ≤ 3 := by linarith [h1, h2]
have h4 : (x + y + z)^2 ≤ 9 := by nlinarith [h1, h2]
have h5 : x + y + z ≤ 3 := by nlinarith [sq_nonneg (x + y + z - 3), h4]
have h6 : x^2*y^2 + y^2*z^2 + z^2*x^2 ≤ 3 := by
  nlinarith [sq_nonneg (x^2 - y^2), ..., h]
have h7 : (x^2 + y^2 + z^2) * (x^2*y^2 + y^2*z^2 + z^2*x^2)
    ≥ 9 * x^2 * y^2 * z^2 := by
  nlinarith [mul_nonneg (sq_nonneg x) (sq_nonneg (y^2 - z^2)), ...]
have hP : 0 ≤ x^2*y^2 + y^2*z^2 + z^2*x^2 := by
  nlinarith [sq_nonneg (x*y), sq_nonneg (y*z), sq_nonneg (z*x)]
have h_SP : (x^2 + y^2 + z^2) * (x^2*y^2 + y^2*z^2 + z^2*x^2) ≤ 9 := by
  nlinarith [mul_nonneg hP (sub_nonneg.mpr h1), ...]
have h8 : x^2 * y^2 * z^2 ≤ 1 := by nlinarith [h7, h_SP]
have h9 : x * y * z ≤ 1 := by nlinarith [sq_nonneg (x*y*z - 1), h8]
exact ⟨h5, h3, h9⟩
\end{lstlisting}

\textbf{(4)~Compared.} The released reference proof splits the conjunction first and
discharges all three parts with one shared list of auxiliary inequalities. Six of
its nine inequalities appear nowhere in the accepted proof.
\begin{lstlisting}[style=lean]
refine' ⟨_, _, _⟩
all_goals nlinarith [sq_nonneg (x - y), sq_nonneg (y - z), sq_nonneg (z - x),
  sq_nonneg (x + y), sq_nonneg (y + z), sq_nonneg (z + x), h,
  sq_nonneg (x - 1), sq_nonneg (y - 1), sq_nonneg (z - 1)]
\end{lstlisting}

\textbf{(5)~Outcome over three runs.} Zero-shot closes zero runs, relevant retrieval
closes two and random retrieval closes one.
\end{ExampleBox}

\end{document}